\documentclass[runningheads]{llncs}

\usepackage[a4paper, total={165mm, 249mm}]{geometry}

\usepackage[T1]{fontenc}
\usepackage{graphicx}
\usepackage{microtype}

\usepackage{subcaption}
\usepackage{booktabs} 

\usepackage{amsmath}
\usepackage{amssymb}
\usepackage{mathtools}

\usepackage[pagebackref,breaklinks,colorlinks]{hyperref}

\usepackage[capitalize]{cleveref}
\crefname{section}{Sec.}{Secs.}
\Crefname{section}{Section}{Sections}
\Crefname{table}{Table}{Tables}
\crefname{table}{Tab.}{Tabs.}

\usepackage{comment}

\usepackage[textsize=tiny]{todonotes}
\usepackage[acronym]{glossaries}

\definecolor{fig1yellow}{HTML}{FDE724}
\definecolor{fig1purple}{HTML}{440154}
\newacronym{gb}{GB}{Guided Backpropagation}
\newacronym{gi}{GI}{Gradient $\times$ Input}
\newacronym{ig}{IG}{Integrated Gradients}
\newacronym{lrp}{LRP}{Layer-wise Relevance Propagation}
\newacronym{mse}{MSE}{Mean-Squared Error}
\newacronym{ssim}{SSIM}{Structural Similarity Index Measure}

\begin{document}
\title{Shortcomings of Top-Down Randomization-Based Sanity Checks for Evaluations of Deep Neural Network Explanations}
\titlerunning{Shortcomings of Sanity Checks for Evaluation of Explanations}

\author{Alexander Binder\inst{1,2}\orcidID{0000-0001-9605-6209} \and
Leander Weber\inst{3} \and
Sebastian Lapuschkin\inst{3}\orcidID{0000-0002-0762-7258} \and
Gr{\'e}goire Montavon\inst{4,5}\orcidID{0000-0001-7243-6186}\and
Klaus-Robert M{\"u}ller\inst{5,6,7,8} \and 
Wojciech Samek\inst{3,5,6}\orcidID{0000-0002-6283-3265}}

\authorrunning{Binder et al.}

\institute{ICT Cluster, Singapore Institute of Technology (SIT), Singapore 138683 \and
University of Oslo, 0316 Oslo, Norway \and
Fraunhofer Heinrich Hertz Institute, 10587 Berlin, Germany\and 
Freie Universit{\"a}t Berlin, 14195 Berlin, Germany\and
BIFOLD -- Berlin Institute for the Foundations of Learning and Data, 10587 Berlin, Germany\and 
Technische Universit{\"a}t Berlin, 10587 Berlin, Germany\and 
Korea University, Seoul 136-713, Republic of Korea\and 
Max Planck Institut für Informatik, 66123 Saarbr{\"u}cken, Germany}

\maketitle              

\begin{abstract}
While the evaluation of explanations is an important step towards trustworthy models, it needs to be done carefully, and the employed metrics need to be well-understood. Specifically model randomization testing is often overestimated and regarded as a sole criterion for selecting or discarding certain explanation methods. To address shortcomings of this test, we start by observing an experimental gap in the ranking of explanation methods between randomization-based sanity checks \cite{DBLP:conf/nips/AdebayoGMGHK18} and model output faithfulness measures (e.g.\ \cite{DBLP:journals/tnn/SamekBMLM17}). We identify limitations of model-randomization-based sanity checks for the purpose of evaluating explanations. 
Firstly, we show that uninformative attribution maps created with zero pixel-wise covariance easily achieve high scores in this type of checks. Secondly, we show that top-down model randomization preserves scales of forward pass activations with high probability. That is, channels with large activations have a high probility to contribute strongly to the output, even after randomization of the network on top of them.  
Hence, explanations after randomization can only be expected to differ to a certain extent. This explains the observed experimental gap. In summary, these results demonstrate the inadequacy of model-randomization-based sanity checks as a criterion to rank attribution methods. 
\end{abstract}

\section{Introduction}
\label{sec:intro}

Parallel to the progressively astounding performances of machine learning techniques, especially deep learning methods, in solving even the most complex tasks, the transparency, trustworthiness, and lack of interpretability of these techniques has increasingly been called into question \cite{DBLP:journals/csr/HuangKRSSTWY20,DBLP:journals/corr/abs-2103-10689,Lapuschkin2019}. As potential solutions to these issues, a vast number of XAI methods have been developed in recent years \cite{DBLP:journals/pieee/SamekMLAM21}, that aim to explain a model's behavior, for instance, by (locally) attributing importance scores to features of singular input samples, indicating how (much) these features influence a specific model decision \cite{simonyan2013deep,DBLP:journals/corr/SundararajanTY17,GuidebackPropagation:springenberg2014striving,bach2015pixel}. 
However, the scores obtained for different attribution map methods tend to differ significantly, and the question arises how well each explains model decisions. This is generally not answered easily, as there are a number of desirable properties proposed to be fulfilled by these attributions, such as localization on relevant objects \cite{DBLP:journals/ijcv/ZhangBLBSS18,DBLP:conf/fuzzIEEE/Arias-DuartPGG22,DBLP:journals/inffus/ArrasOS22} or faithfulness to the model output
\cite{DBLP:journals/tnn/SamekBMLM17,DBLP:conf/accv/AgarwalN20,DBLP:conf/ijcai/BhattWM20}, among others, with several quantitative tests having been proposed for each.

In parallel to these empirical evaluations, several works have proposed that explanations should fulfill a certain number of `axioms' or `unit tests' \cite{DBLP:journals/corr/SundararajanTY17,Kindermans2019,Montavon2019gradvspropagationbasedcomparison,DBLP:conf/nips/AdebayoGMGHK18}, which need to hold universally for a method to be considered good or valid. We place our focus on the  model-randomization-based sanity checks \cite{DBLP:conf/nips/AdebayoGMGHK18}, which state that the explanation should be sensitive to a random permutation of parameters at one or more layers in the network. Specifically, the authors proposed to apply measures such as \gls{ssim} \cite{wang2004image} between attribution maps obtained from the original model and a derived model for which the top-layers are randomized. The idea is to require that methods used to compute attribution maps should exhibit a large change when the neural network model  --- i.e., its defining/learned parameter set --- is randomized from the top. The authors of~\cite{DBLP:conf/nips/AdebayoGMGHK18,DBLP:conf/icml/SixtGL20} suggested to discard attribution map methods which perform poorly under this test --- i.e., have a high \gls{ssim} measure between attributions obtained with the original and the randomized model --- under the assumption that those XAI methods are not affected by the model's learned parameters.

However, we observe a significant experimental gap between top-down randomization checks when used as an evaluation measure, and occlusion-based evaluations of model faithfulness such as region perturbation \cite{DBLP:journals/tnn/SamekBMLM17}. Concretely, \gls{gb}~\cite{GuidebackPropagation:springenberg2014striving} and \gls{lrp}~\cite{bach2015pixel} exhibit low randomization scores under the first type of measure and yet clearly outperform several gradient-based methods in occlusion-based evaluations. We are interested to resolve this discrepancy.

We identify two shortcomings of top-down randomization checks when used as a measure of explanation quality.
Firstly, we show that uninformative attribution maps created with zero pixel-wise covariance --- e.g., attribution maps generated from random noise --- easily achieve high scores in top-down randomization checks. Effectively, this makes top-down randomization checks favor attribution maps which are affected by gradient shattering noise \cite{DBLP:conf/icml/BalduzziFLLMM17}.

Secondly, we argue that the randomization-based sanity checks may always reward explanations that change under randomization, even when such randomizations do not affect the output of the model (and its invariances) significantly. Such invariance to randomization may result, e.g., from the presence of skip connections in the model, but also due to the fact that randomization may be insufficient to strongly alter the spatial distribution of activations in adjacent layers, something that we explain by the multiplicity and redundancy of positive activation paths between adjacent layers in ReLU networks.

Along with our contributed theoretical insights and supporting experiments, the present note warns against an unreflected use of model-randomization-based sanity checks as a sole criterion for selecting or dismissing a particular attribution technique, and proposes several directions to enable a more precise and informative use of randomization-based sanity checks for assessing how XAI performs on practical ML models.

\subsection{Related work}

\paragraph{Evaluating Attributions.}
Comparing different attribution methods qualitatively is not sufficiently objective, and for that reason, a vast number of quantitative tests have been proposed in the past in order to measure explanation quality, focusing on different desirable properties of attributions. Complexity tests \cite{DBLP:conf/icml/ChalasaniC00J20,DBLP:conf/ijcai/BhattWM20,DBLP:journals/corr/abs-2007-07584} advocate for sparse and easily understandable explanations, while robustness tests \cite{DBLP:conf/nips/Alvarez-MelisJ18,DBLP:conf/ijcai/BhattWM20,DBLP:journals/dsp/MontavonSM18} measure how much attributions change between similar samples or with slight perturbations to the input. Under the assumption of an available ground truth explanation (e.g., a segmentation mask localizing the object(s) of interest), localization tests \cite{DBLP:journals/ijcv/ZhangBLBSS18,DBLP:conf/fuzzIEEE/Arias-DuartPGG22,DBLP:journals/inffus/ArrasOS22} ask for attributed values to be concentrated on this ground truth area. Faithfulness tests \cite{DBLP:journals/tnn/SamekBMLM17,DBLP:conf/nips/Alvarez-MelisJ18,DBLP:conf/ijcai/BhattWM20} compare the effect of perturbing certain input features on the model's prediction to the values attributed to those features, so that optimally perturbing the features with the largest attribution values also affects the model prediction the most. Model randomization tests \cite{DBLP:conf/nips/AdebayoGMGHK18}, which are the main focus of this work, progressively randomize the model, stating that attributions should change significantly with ongoing randomization.

\paragraph{Caveats of Model Randomization Tests.}
The authors of \cite{DBLP:conf/nips/AdebayoGMGHK18} find that a large number of attribution methods seems to be invariant to model parameters, as their explanations do not change significantly under cascading model randomization. However, various aspects of these sanity checks have recently been called into question: 
For instance, these tests were performed on unsigned attributions. Specifically for \gls{ig}~\cite{DBLP:journals/corr/SundararajanTY17}, \cite{DBLP:journals/corr/abs-1806-04205} show that if the signed attributions are tested instead, this method suddenly passes cascading model randomization instead of failing. This indicates that some of the results obtained in \cite{DBLP:conf/nips/AdebayoGMGHK18} for attribution methods where the sign carries meaning may be skewed due to the employed preprocessing.
Furthermore, \cite{DBLP:journals/corr/abs-2110-14297} argue for the distribution-dependence of model-randomization based sanity checks. The authors demonstrate that some methods seem to fail the sanity checks in \cite{DBLP:conf/nips/AdebayoGMGHK18} due to the choice of task, rather than invariance to model parameters. A similar observation is made by \cite{DBLP:journals/corr/abs-2106-07475}, who find that the same attribution methods can perform very differently under model randomization sanity checks when the model and task are varied. Note that the underlying assumption of \cite{DBLP:conf/nips/AdebayoGMGHK18} --- that ``good'' attribution methods should be sensitive to model parameters --- is not called into question here. Rather, we posit that methods can fail the model randomization sanity checks for other reasons than invariance to model parameters. 

\section{Observation: The Gap between Randomization-based Sanity Checks and Measures of Model Faithfulness}
\label{sec:faithfulness_comp}

Based on the assumption that a good explanation should attribute the highest values to the features that most affect a model's predictions, occlusion-type measures of model faithfulness \cite{DBLP:journals/tnn/SamekBMLM17,DBLP:conf/accv/AgarwalN20,DBLP:conf/ijcai/BhattWM20} aim to quantify explanation quality by measuring the correlation between attribution map scores and changes of the model prediction under occlusion.

As such, these tests progressively randomize the data, and can thus be understood as complementary to model-randomization-based sanity checks, which progressively randomize the model. Consequently, model-randomization-based sanity checks depart towards the implausible due to partially randomized prediction models, while occlusion-based testing departs towards the implausible due to partially modified and outlier-like images. As both types of test apply the same intuition of increasing randomization to different variables (model and data) that (should) influence attribution maps, it is meaningful to compare their results, and determine whether both tests agree in terms of explanation quality.

In the following, we therefore empirically compare the scores measured by randomization-based sanity checks to the respective scores measured by faithfulness testing, for several methods. We use a variant of occlusion in the spirit of \cite{DBLP:conf/accv/AgarwalN20} which replaces a region with a blurred copy to stay closer to the data manifold. Details on our experimental setup can be found in the Supplement (Section \ref{ssec:experiment_details:faithfulness_comp}).

As already known from \cite{DBLP:conf/nips/AdebayoGMGHK18} (and also shown in the Supplement, see Figures \ref{fig:ssim2} and \ref{fig:mse}), Guided Backpropagation \cite{GuidebackPropagation:springenberg2014striving} performs poorly under model randomization-based sanity checks when compared to three gradient-based attribution methods, namely the Gradient itself, \gls{gi} and \gls{ig}. However, when measuring model faithfulness by a modified iterative occlusion test similar to \cite{DBLP:journals/tnn/SamekBMLM17} on the attribution maps, we find that the same \gls{gb}, and also several \gls{lrp} variants outperform the Gradient, Gradient $\times$ Input and Integrated Gradient substantially, as can be seen in Figure \ref{fig:perturb_gbpvsgrad_k15} and in the Supplement in Section \ref{ssec:morefaithfulessexperiments}.

\begin{figure*}[!ht]
\centering     

\includegraphics[width=0.99\linewidth]{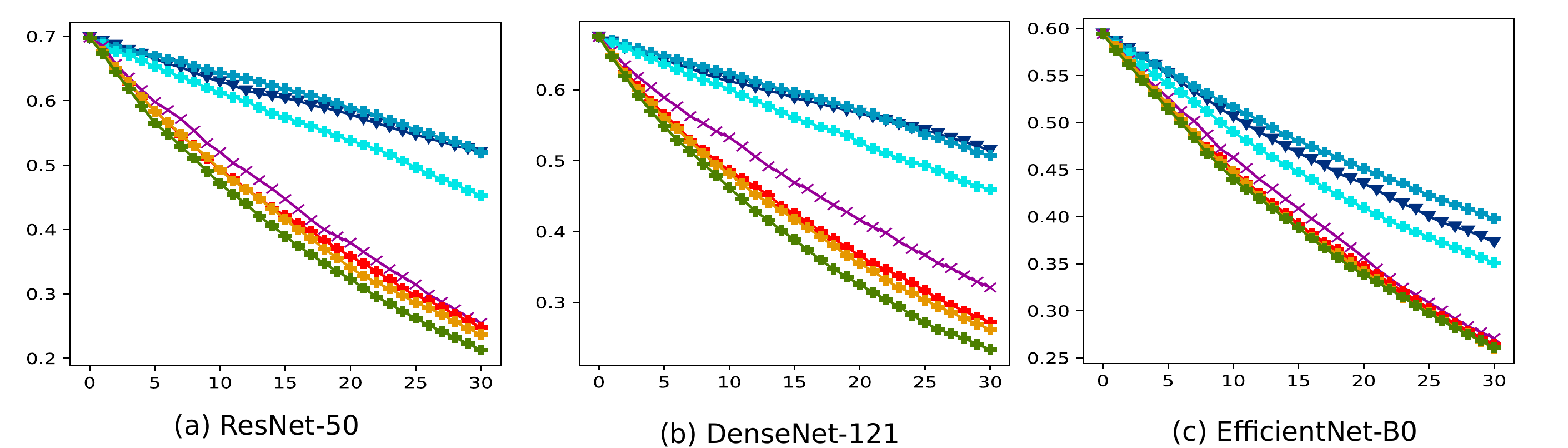}

\caption{\label{fig:perturb_gbpvsgrad_k15} Results of model faithfulness via occlusion testing, by measuring the correlation to iterative occlusion with a kernel size of 15. The comparison shows the Gradient, \glsdesc{gi}, \glsdesc{ig}, \glsdesc{gb} and several variants of \gls{lrp}. The occlusion is performed by taking patches from a blurred copy of the original image. The figure shows the softmax scores. Legend is given in Figure \ref{fig:ssim}. \textit{Lower is better.}}
\end{figure*}

\begin{figure*}[ht]
\centering     

\includegraphics[width=0.99\linewidth]{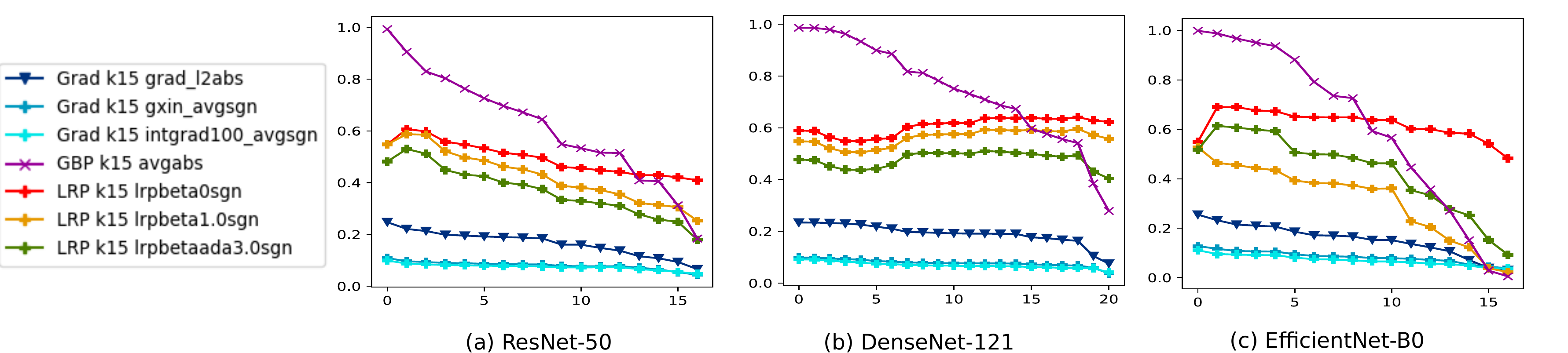}
\caption{\label{fig:ssim} Results of top-down model randomization-based sanity checks with \gls{ssim} after normalization of attribution maps by their second moment. Note that the model randomization experiment uses no kernel size. \textit{Lower is better.}}
\end{figure*}

Due to the conceptual parallels between both tests discussed above, we find this extreme divergence surprising, and are interested in resolving this gap. Therefore, we will investigate the underlying reasons for this theoretically and experimentally in the following sections.

\section{The Sensitivity of SSIM Minimization Towards Noise}
\label{sec:ssimsens}
The model-randomization-based sanity checks proposed by \cite{DBLP:conf/nips/AdebayoGMGHK18} use \gls{ssim} as a measure of distance between attribution maps. As we will demonstrate in this section, \gls{ssim} (and, by extension, several other distance measures, see Supplement Sections \ref{ssec:spearmansens} and \ref{sec:msesens}) may be flawed in this application, with randomly generated attributions scoring optimally.
We consider a setup where we use two different models, yielding two different attribution maps $A$ and $B$. The following considerations apply to patches of the two attribution maps or whole attribution maps.

We can identify a fundamental issue: The \gls{ssim} between any two attribution maps can be minimized by a statistically uncorrelated random attribution process. This is due to the reason that the \gls{ssim} contains a product where one term relies on a covariance between two patches, see e.g.~Equation 6 in \cite{nilsson2020understanding}, which is reproduced here: 
\begin{equation}
\label{eq:ssimeq}
    \frac{2\mu_A \mu_B +C_1}{ \mu^2_A + \mu^2_B +C_1  } \frac{2\sigma_{AB} +C_2}{ \sigma^2_A + \sigma^2_B +C_2}\, .
\end{equation}

In the above term, $\mu_A$, $\mu_B$ and $\sigma^2_A$, $\sigma^2_B$ are the per-patch means and variances for one patch location computed for two different input attribution maps $A$ and $B$, $\sigma_{AB}$ their covariance. $C_1$ and $C_2$ are constants depending on the possible input range of $A$ and $B$, e.g.~$[0,1]$ or $[0,255]$.

In the following, we will consider attribution maps within the framework of random variables.
The next theorem is applied to patches of two attribution maps $A$, $B$ coming from different prediction mappings, such as those obtained by a model and a partially randomized model. The patches are extracted at the same position of an image. 
\begin{theorem}
\label{Theor.1}
Consider the set of all random variables with expected means $\mu_A$, $\mu_B$ for each image patch being fixed 
and with non-negative expected covariance for each patch $\sigma_{AB} \ge 0$. 

Then the expected \gls{ssim} absolute value is minimized by a random variable with zero covariance. In particular, an upper bound on the minimum is given by \begin{align}
\frac{C_2}{\sigma^2_A + \sigma^2_B +C_2} \, .
\end{align}

\end{theorem}

The proof is in the Supplement in Section \ref{ssec:app:proofofthm1}. This theorem has two consequences. 

Firstly, even if we question the requirements of the theorem and thus allow negative patch correlations $\sigma_{AB} < 0$, the observation remains valid that we can obtain \emph{very small} expected absolute values of the \gls{ssim} measure by using any randomized attribution map which is statistically independent over pixels of input images and therefore not informative. 

Secondly, the proof of Theorem \ref{Theor.1} is not affected by division of the term $\sigma_{AB}$ by constants. Consequently, when using normalization on the attribution maps, the result from Theorem \ref{Theor.1} still holds that attributions with zero patch-wise correlation attain very low scores among all normalized attribution maps.

Interestingly, this explains why certain gradient-based methods with rather noisy attribution maps pass this type of model randomization-based sanity checks with the best scores in the sense of lowest \gls{ssim} values. Gradients are known for ReLU-networks to have statistics which resemble noise processes, as has been shown in \cite{DBLP:conf/icml/BalduzziFLLMM17}. This carries over to Gradient $\times$ Input and to a lesser degree to smoothed versions like Integrated Gradient \cite{DBLP:journals/corr/SundararajanTY17} and SmoothGrad \cite{DBLP:journals/corr/SmilkovTKVW17}.

Theorem \ref{Theor.1} has one important consequence: One cannot disentangle the effects of model randomization from the amount of noise in an attribution process in model randomization sanity checks. Therefore it is problematic to use this type of model-randomization-based sanity check to compare or rank different attribution maps against each other.

\section{Randomization Leaves the Model and Explanations Partly Unchanged}

The section above has highlighted that explanations may score highly in the sanity check due to including further random factors in the explanation, which contradicts the principle that an explanation should faithfully depict the function to explain and not a random component. This concerns the measurement process after randomization. 

In this section, we review the top-down randomization process itself. We will explain why it actually makes sense to underperform in model-based randomization checks, contrary to a first glance intuition.

Specifically, we will show that certain activation statistics of the network are only mildly affected by top-down randomization and thus cause low-noise explanations before and after randomization to retain some similarity. 

\subsection{Preservation of Irrelevance in Explanations}

We first start with an empirical observation that features found to be irrelevant for a given task tend to remain irrelevant after randomization. Our experiment is based on torchvision's VGG-16 pretrained model, where we keep the mapping from the input to layer $12$ unchanged and randomize the remaining layers. We apply LRP with the $z^\mathcal{B}$-rule in the first layer, redistribution in proportion to squared activations in the pooling layers, LRP-$\gamma$ in the convolution layers with layer-wise exponential decay from $\gamma=1.0$ to $\gamma=0.01$, and LRP-0 in the dense layers. We inspect in Fig.\ \ref{fig:impl_invariance-gen} (right) explanations produced before and after randomization. 

We observe that many spatial structures are retained before and after randomization, specifically, relevant or negatively contributing pixels are found before and after randomization on the facial and hat features, on the outline of the fish, on the finger contours, on the flagstick, on the ball, on the hole, etc. Conversely, some features remain irrelevant before and after randomization, e.g.\ the lake surface, the skin and the grass. Such similarities lead to similarity scores before and after randomization that remain significantly above zero, especially if considering heatmaps absolute scores.

We now provide a formal argument showing that for an explanation to be faithful, some irrelevant features must \textit{necessarily} remain irrelevant after randomization, thereby raising the similarity score. Let us denote by $\theta_R$ the parameters that are randomized and write the model as a composition of the non-randomized and the randomized part:
\begin{align}
   f(x,\theta_R) = g(\phi(x),\theta_R) \ . \label{eq:twoparts}
\end{align}
The function is depicted in Figure \ref{fig:somerandomizedlayers}. The first part $\phi$ contains the non-randomized layers (and can be understood as a feature extractor). The second part $g$ contains the randomized layers (and can be interpreted as the classifier). We make the following two observations:

\begin{figure}[t]
\centering
\includegraphics[width=0.6\linewidth]{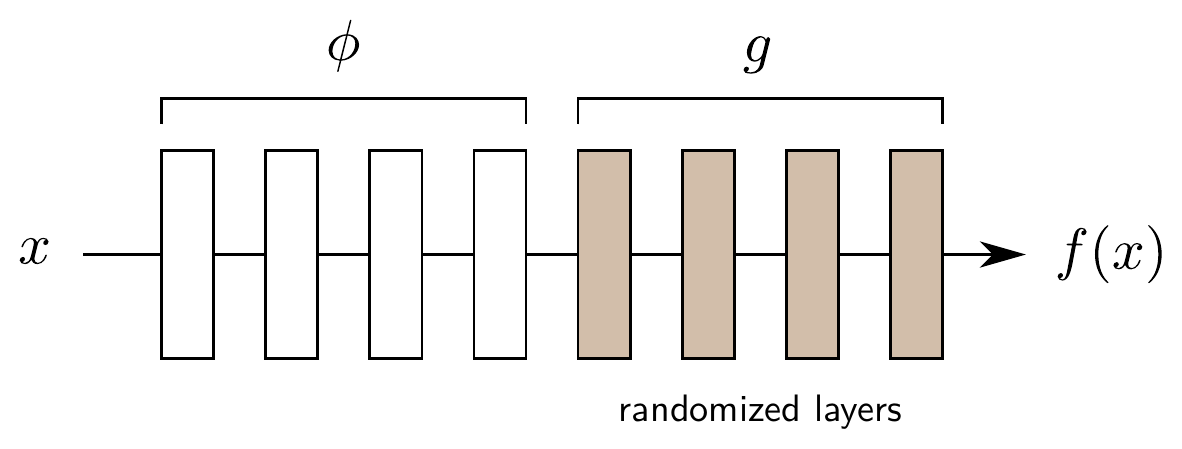}
\caption{\label{fig:somerandomizedlayers} Diagram of a neural network where the top few layers have been randomized (shown in brown).}
\end{figure}
\begin{figure*}[t!]
    \centering
    \includegraphics[width=.95\textwidth]{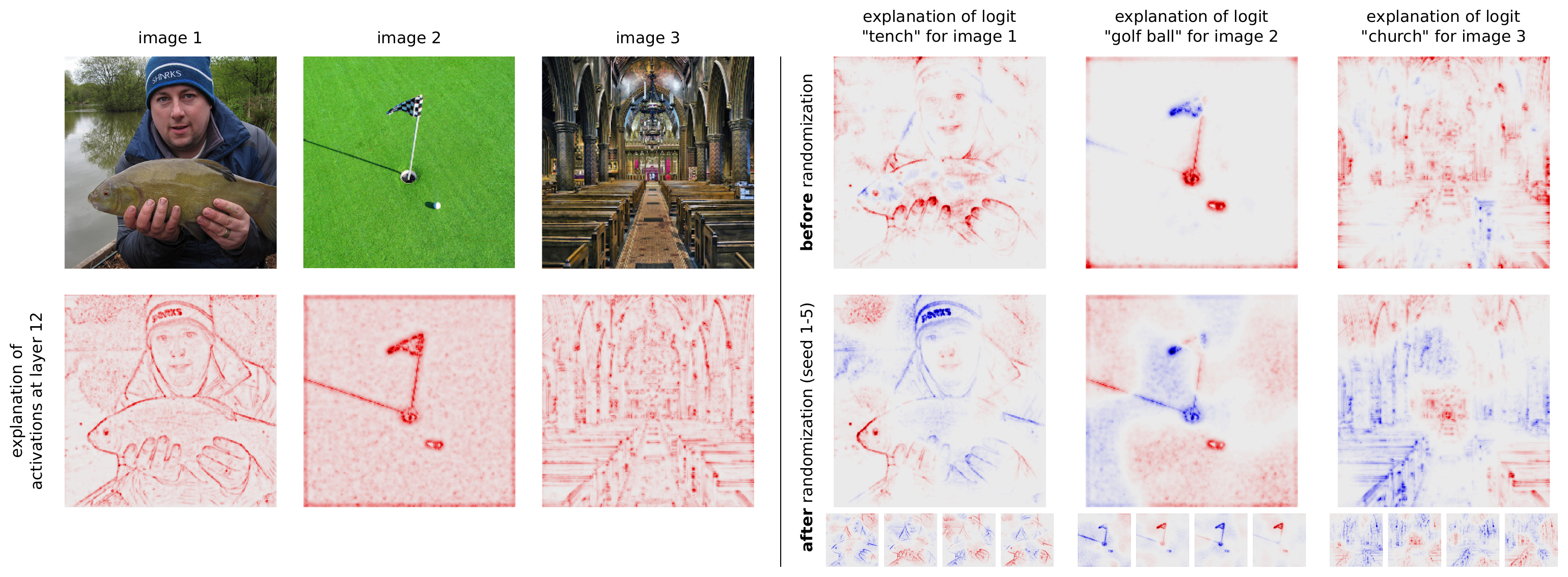}
    \caption{Experiment where one randomizes parameters of torchvision's VGG-16 pretrained model between layer $12$ and the model output, and compute LRP explanations (results shown for 5 different seeds). Explanations of the neural network output before and after randomization for the true class are shown on the right. Explanation of activations at layer 12 are shown on the bottom left.}
    \label{fig:impl_invariance-gen}
\end{figure*}

\begin{enumerate}
    \item If the function $\phi$ is does not respond to some input feature $x_i$, then $g \circ \phi$ also should not respond to $x_i$ (no matter whether the function $g$ is the classifier or its randomized variant),
    \item If $g \circ \phi$ does not respond to $x_i$ then an attribution technique should reflect this lack of response by assigning a score $0$ to that feature. We note that this property of attributing low relevance input features to which the model does not respond is present in common explanation methods, for example, methods such as \gls{ig}, where the gradient occurs as a multiplicative factor, LRP-type explanations, where relevance propagates mostly along connections with non-zero weights, or explanations derived from axioms such as the Shapley value whose `null-player' axiom also relates explanation properties to model unresponsiveness.
\end{enumerate}
These two observations can be summarized in the following logical clause:
\begin{align}
&\phi(x) \text{ unresponsive to } x_i\nonumber\\
&\qquad \Rightarrow \forall g: g \circ \phi(x) \text{ unresponsive to } x_i\nonumber\nonumber\\
& \qquad \qquad \Rightarrow \forall g: \mathcal{E}_i\{g \circ \phi(x)\} \text{~small},
\end{align}
where $\mathcal{E}_i\{\cdot\}$ denotes the relevance of feature $x_i$ for explaining the prediction given as argument.
In other words, one should expect that any function $g$ (randomized or not) built on $\phi$ shares a similar pattern of low relevances, and such a pattern originates from the lack of response of $\phi$ to certain input features. Therefore, we conclude that a top-down randomization process as performed in \cite{DBLP:conf/nips/AdebayoGMGHK18} can only alter explanations to a limited extent, and only a less faithful (e.g.\ noisy) explanation would enable further improvement w.r.t.\ the top-down randomization metric.

\medskip

To verify that the explanation structure is indeed to some extent controlled by $\phi$, we compute explanations directly at the output of the function $\phi$ (sum of activations) and show the results in Figure \ref{fig:impl_invariance-gen} (bottom left). We observe a correlation between feature relevance w.r.t.\ those activations and feature relevance w.r.t.\ the model output. For example, the lake, the grass, or more generally uniform surfaces are already less relevant at the output of $\phi$, and continue to be so when considering the output of $g$. This is consistent with our theoretical argument that feature irrelevance of some features to classifier output $g$ is inherited to a significant degree from the feature map $\phi$.

\subsection{Preservation of a Baseline Explanation}
\label{ssec:additivebaseline}

\begin{figure*}[t!]
    \centering
    \includegraphics[scale=.5]{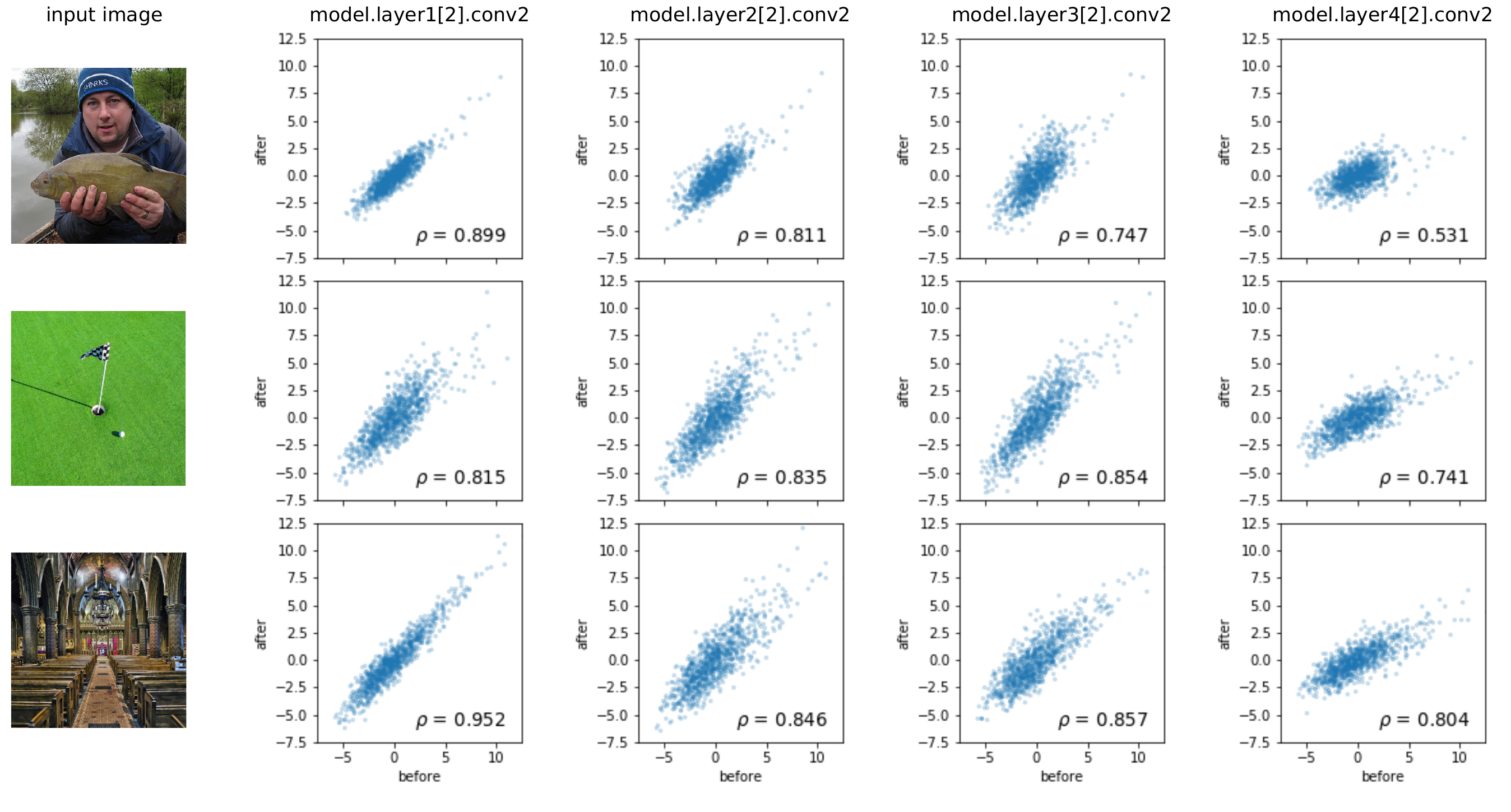}
    \caption{Effect of randomization on output logits on a ResNet-34 model for three images from ImageNet. Each point in the scatter plot is a logit for a particular class before and after randomization. Columns correspond to different layers being randomized.}
    \label{fig:impl_invariance}
\end{figure*}

We show that for certain neural network architectures, specifically architectures that contain skip connections, a faithful explanation must further retain an additive baseline component before and after randomizations. We first demonstrate the presence of such additive component on the popular ResNet~\cite{Resnet:he2016deep} model and then propose an explanation for its necessity. The ResNet is structured as a sequence of multiple modules where each module is structured as a sequence of parameterized layers, equipped with skip connections. The skip connections enable to better propagate the forward and backward signal as they simply replicate the activation and gradients from layer to layer.

Fig.\ \ref{fig:impl_invariance} shows for a ResNet-34 model and the same images as before how randomizing weights at some layer affects logit scores before and after randomization. Each point in the scatter plot is one of the 1000 class logits. We observe significant correlation between the logit before and after randomization. This suggests that the model remains unchanged to a large extent and a faithful explanation should reflect such lack of change by producing a similar explanation.

Corresponding explanations are shown in Figure \ref{fig:impl_invariance_heatmaps} for the logit associated to the true class: When randomizing ``\texttt{layer4.2.conv2}'' of ResNet-34, the explanation remains largely the same (cf.\ column 2 and 5). The LRP explanation technique enables to assess contribution of different components of the neural network, and in our case, we can identify the role of the skip connection and the weighted path (cf.\ columns 3, 4, 6, 7). Interestingly, the explanation component that passes through the skip connection remains practically unchanged after randomization, thereby faithfully reflecting the lack of change at the output of the network (cf.\ Figure \ref{fig:impl_invariance}). The (weaker) contribution of the weighted path is strongly affected by randomization but its addition does not affect the overall explanation significantly.

\medskip

We propose a formal argument that predicts the presence (and necessity) of an additive component for a broader range of faithful explanation methods, beyond LRP. Consider the simple architecture drawn in Fig.\ \ref{figure:diagram-resnet} (top) that mimics parts of a ResNet: a feature extractor, a skip connection layer, and a few top-layers.
\begin{figure}[t]
\centering
\includegraphics[width=0.6\linewidth]{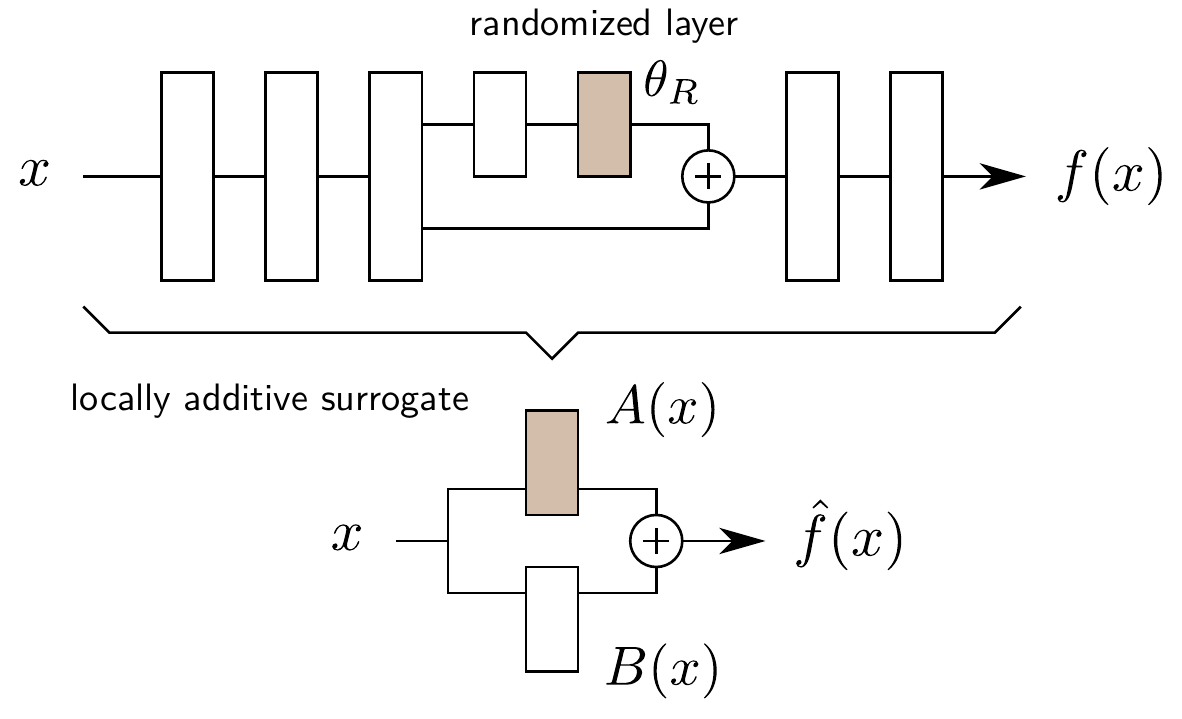}
\caption{Top: ResNet-like structure where only one branch contains (randomizable) parameters at a particular layer. Bottom: Additive surrogate of the original model.}
\label{figure:diagram-resnet}
\end{figure}
\begin{figure*}[t!]
    \centering
    \includegraphics[width=0.9\linewidth]{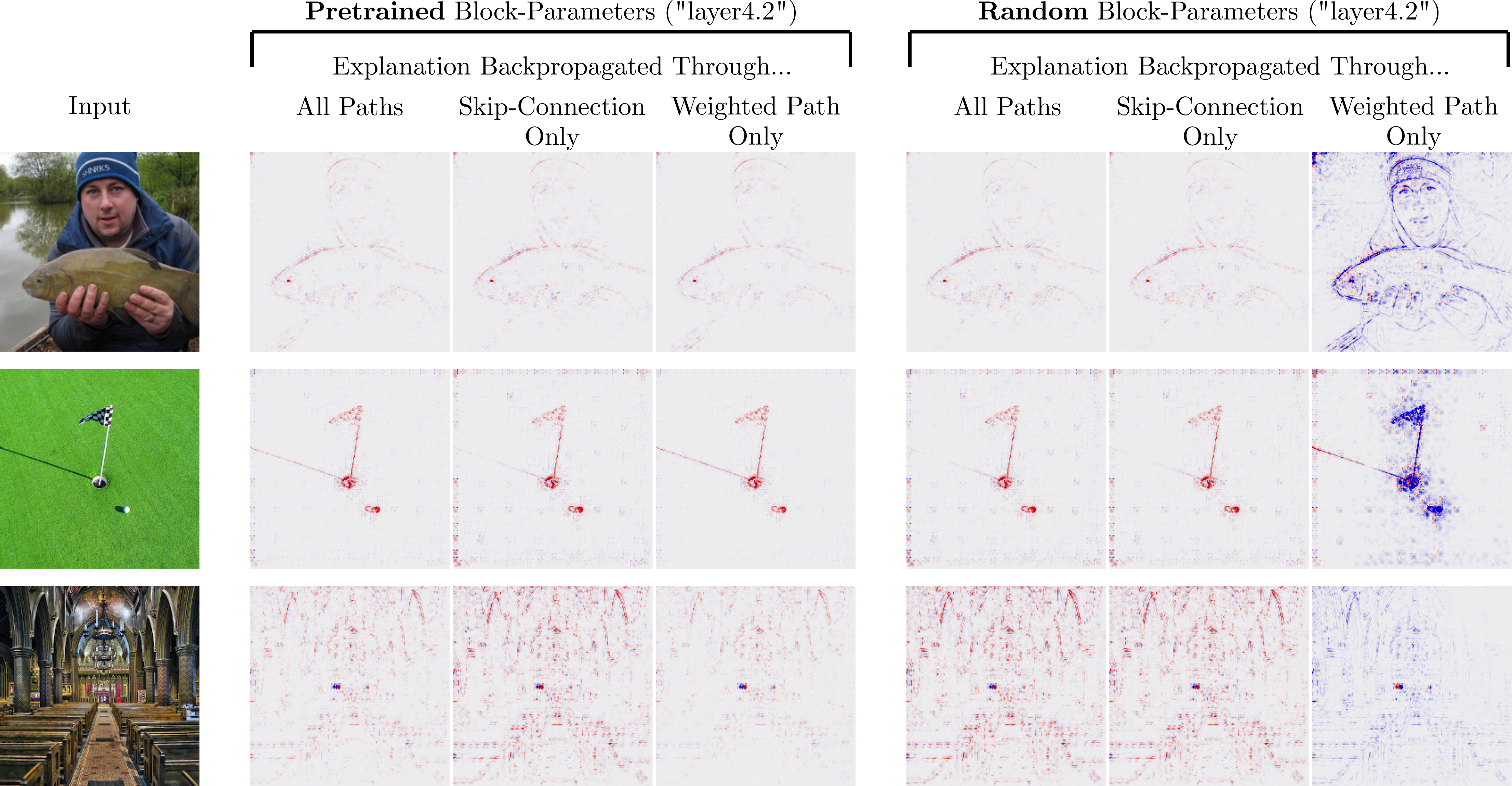}
    \caption{Effect of parameter randomization of "layer4.2.conv2" on ResNet-34 Explanations.
     To generate explanations, LRP-$\beta$ was applied in convolution layers, and LRP-$\varepsilon$ in dense layers. Despite random parameter re-initialization, only the part of the explanation that is propagated through the weighted path changes significantly. Due to the skip connections being unaffected, the total explanations barely change despite randomization.}
    \label{fig:impl_invariance_heatmaps}
\end{figure*}
Locally approximating the top layers as a linear model (and verifying that the approximation holds under a sufficient set of perturbations of the input $x$), then one can decompose this approximated model in two terms, one that depends on the randomized parameter $\theta_R$, and another term that is constant w.r.t.\ $\theta_R$.
\begin{align}
\hat{f}(x;\theta_R) = A(x;\theta_R) + B(x)
\end{align}
(cf.\ Fig.\ \ref{figure:diagram-resnet}, bottom).
In this model, randomization only affects the first component and thus preserves some of the original logit information. This is what we observe empirically in Fig.\ \ref{fig:impl_invariance} through high correlation scores of logit before and after randomization. If we further assume usage of an explanation technique satisfying the linearity property, then the explanation of the surrogate $\hat{f}$ decomposes as:
\begin{align}
\mathcal{E}\{ \hat{f}(x;\theta_R) \} = \mathcal{E}\{ A(x;\theta_R) \} + \mathcal{E}\{ B(x) \}
\end{align}
i.e.\ an explanation component that is affected by randomization and another explanation component that remains constant. This constancy under randomization prevents that optimal scores in terms of a single-layer randomization-based sanity check metric are achieved. Hence, our analysis predicts that attempts to score higher in the sanity check metric would require degrading the faithfulness of the explanation (e.g.\ by introduction of noise in the explanation, or by spuriously removing the additive component).

\subsection{Probabilistic Preservation of Highly Activated Features in the Unrandomized Feature Layers}
\label{ssec:preservationofhighlyactivatedfeatures}

In this section we show that with high probability over draws of random parameters $\theta_R$ in Equation \eqref{eq:twoparts}, regions of $\phi(x)$ with high activations will contribute highly to the output $f$, even when its value changes due to randomization in $g$. Unlike in the previous section, this holds for \emph{any} ReLU network.

This observation can be explained when considering how activations are propagated to the next layer in a randomized network. One can show that small activations have a rather low probability to obtain the same average contribution to the output of a neuron as large activations, when being weighted in a linear combination with zero-mean normal weights. 
This statement is formulated for a single neuron in Theorem \ref{thm:limitedeffectofsmallactivations}.

\begin{theorem}[Low probability for small activations to achieve the same average contribution to the output as large activations]\label{thm:limitedeffectofsmallactivations}

Suppose we have two sets of non-negative activations, $X_{L}$ and $X_{S}$ such that the activations of one set are by a factor of at least $K$ larger than of the other set:
\begin{align}
    \min_{x_l \in X_L} x_l \ge K \max_{x_s \in X_S} x_s
\end{align}

Then the probability under draws of zero-mean normal weights $w \sim N(0,\sigma^2)$ that the summed contribution of neurons in $X_S$ surpasses the summed contribution of neurons in $X_L$, that is 
\begin{align}0<\sum_{ x_l \in X_L} w_l x_l  \le \sum_{ x_s \in X_S} w_s x_s\ ,
\end{align}
is the tail-CDF $P(Z \ge K)$ of a Cauchy-distribution with parameter $\gamma = \sqrt{\frac{|X_S|}{|X_L|}}$ and input value of at least $K$.
\end{theorem}
The proof of Theorem \ref{thm:limitedeffectofsmallactivations} is in Supplement Section \ref{ssec:app:proofofthm2}. For probability estimates based on activation statistics of trained networks see Section \ref{ssec:probforwardpassovertaking} of the supplement.

To note, Theorem \ref{thm:limitedeffectofsmallactivations} is independent of any explanation method used.  It is a statement about the preservation of relative scales of forward pass activations. It says that even though the function output value itself changes substantially under randomization, channels with large activation values still contribute highly to the output.

This effect has an impact on explanation value scales: In ReLU networks, with neurons being modeled as $ y = \max(0, \sum_i w_i x_i +b )$, the differences in contributions of two inputs $w_i x_i$ to a neuron output $y$ in many cases translate to differences in explanation scores $R(x_i)$ which the inputs $x_i$ will receive.

Many explanation methods $R(\cdot)$ satisfy for non-decreasing activations the monotonicity property that if we consider two inputs $x_i$, $x_j$ which have no other connections except to neuron $y$, and the network assigns positive relevance $R(y)>0$ to $y$, then 
\begin{align}
w_ix_i \ge w_j x_j>0  \text{ implies } |R(x_i)| \ge |R(x_j)| \ .\label{eq:posmonotonicityspecial}  
\end{align}
This holds for \glsdesc{gi}, Shapley values, and $\beta$-\gls{lrp}. See Supplement Section \ref{ssec:monotonicityofrelevance} for a proof.

Using the monotonicity property to go from activations to explanations, we can conclude that the probability is low to achieve equally large absolute explanation values $\sum_{x_i \in X_S}|R(w_ix_i)|$ for inputs from the small valued set $X_S$ in Theorem \ref{thm:limitedeffectofsmallactivations}, when compared to $\sum_{x_i \in X_L}|R(w_ix_i)|$ from the set of large values $X_L$.\footnote{If we intend to achieve equal averaged (instead of summed) absolute explanation values $\frac{1}{|X_{S/L}|}\sum_{x_i \in X_{S/L}}|R(w_ix_i)|$, corresponding to two regions $X_S$ and $X_L$ with equal explanation scores, then a version of Theorem \ref{thm:limitedeffectofsmallactivations} holds in which $\gamma_*= \sqrt{\frac{|X_L|}{|X_S|}}$ is inverted. See Supplement Section \ref{ssec:app:proofofthm2} for a proof.}

Therefore, with high probability, explanations are also dominated by regions which have channels in the last unrandomized feature map with large activation values. 

Theorem \ref{thm:limitedeffectofsmallactivations} and the subsequent backward pass argument hold for a single neuron.  We can see that what the theoretical result predicts for a single neuron is consistent with what we observe empirically for the whole network exemplarily in Figure \ref{fig:impl_invariance-gen} and generally in explanations computed with \gls{gb} and \gls{lrp}. The above provides a theoretical justification for the exemplary observations in Figure \ref{fig:impl_invariance-gen}, where one can see salient structures from the input image in the explanation heatmaps after top-down network randomization.

In brief, explanation heatmaps will be dominated with high probability by regions with high activations irrespective of randomization on top, thus showing limited variability under randomization. This has implications regarding the usage of top-down randomization tests to compare attribution methods: 
a higher variability does not imply a better explanation, when it is beyond what can be expected from the dominating high activations in the forward pass.

A further property which is preserved is shown in the Supplement Section \ref{ssec:posdominance}.
The randomization-based sanity check fails to account for these necessary invariances of the model and explanation under randomization. This misalignment 
is particularly strong if testing the effect of randomization on the \textit{absolute} attribution scores instead of the signed attribution scores. The necessity to use signed scores rather than absolute ones, as well as the limited change to the explanation one can expect under randomization of parameters was also emphasized in  \cite{DBLP:journals/corr/abs-1806-04205}.

Given the discrepancy between model faithfulness measures and top-down model-based randomization checks, we remark that model faithfulness testing changes an input sample towards a partially implausible sample. Therefore it is not a perfect criterion. Another drawback is the non-uniqueness of local modifications. Different choices of local modifications will yield different measurements. However, model faithfulness testing assesses a property of explanations for a given trained model.

Model randomization changes a trained model into a predominantly implausible model given the training data. Therefore it is not clear what practical aspect of a \emph{given realistically trained} model top-down model randomization intends to measure. It seems to be unrelated to any use-case in the deployment of a well-trained model. This makes it challenging to suggest an improvement for top-down model-based randomization checks.

In contrast, bottom-up randomization could exhibit different (and ecologically valid) properties, because it removes strongly activated features from a model, which were the starting point in Theorem \ref{thm:limitedeffectofsmallactivations}.

\section{Conclusion and Outlook}

In this paper, we have argued against the practice of using top-down randomization-based sanity checks \cite{DBLP:conf/nips/AdebayoGMGHK18} for the purpose of ranking XAI methods. Our study is motivated by a substantial empirical discrepancy between the similarity scores produced by the randomization approach, and occlusion-based methods for evaluating faithfulness, in particular region perturbation \cite{DBLP:journals/tnn/SamekBMLM17}.

Note that our theoretical and empirical results do not contradict the overall claim of \cite{DBLP:conf/nips/AdebayoGMGHK18} that a perturbation of the parameters of the model should induce a perturbation of the model and its prediction behavior, which in turn should also perturb the explanation. The issue is instead that the similarity score should only be used as a binary test to support the presence or absence of an effect of randomization on the explanation, but not to discriminate between two methods that pass the randomization test.

We have presented two main factors that explain the discrepancy between randomization-based similarity scores and the outcome of input perturbation tests: Firstly, the similarity scores used to measure the effect of randomization can be decreased artificially (and significantly) by introducing noise in the explanation. Such noise can be inherited from the gradient, which is typically highly varying and largely decoupled from the actual prediction for deep architectures.

Secondly, model randomization only alters the prediction behavior to a certain extent, often due to fixed elements in the model such as skip connections or invariances inherited from the lower layers. Hence, a maximally dissimilar explanation after randomization may not account for the partly unchanged prediction behavior of the randomized model.

These factors suggest directions for achieving a better correlation between similarity scores after randomization and evaluations of explanation faithfulness. These include (1) to only measure change w.r.t.~input features to which the model is not invariant to (such features can be identified by attributing intermediate-layer activations to the input layer and retaining only input features with non-zero attribution scores), and (2) to identify the non-baseline component of the function, and only assess whether the explanation of that non-baseline component has been randomized (e.g.~to exclude from the explanation what passes through the skip connections).

Nevertheless, these possible refinements are challenging to characterize formally, or must address the specificity of individual architectures, thereby losing the universality of the original randomization test.

\section*{Acknowledgements}

AB was supported by the SFI Visual Intelligence, project no.~309439 of the Research Council of Norway. KRM was partly supported by the Institute of Information \& Communications Technology Planning \& Evaluation (IITP) grants funded by the Korea government(MSIT) (No. 2019-0-00079, Artificial Intelligence Graduate School Program, Korea University and No. 2022-0-00984, Development of Artificial Intelligence Technology for Personalized Plug-and-Play Explanation and Verification of Explanation). This work was supported in part by the German Ministry for Education and Research (BMBF) under Grants 01IS14013A-E, 01GQ1115, 01GQ0850, 01IS18025A and 01IS18037A. LW, SL and WS were partly supported by the German Ministry for Education and Research (BMBF) under Grants [BIFOLD (01IS18025A, 01IS18037I) and Patho234 (031L0207C)], the European Union's Horizon 2020 research and innovation programme as grant [iToBoS (965221)],  and the state of Berlin within the innovation support program ProFIT as grant [BerDiBa (10174498)].

\bibliography{references}

\begin{thebibliography}{10}\itemsep=-1pt

\bibitem{DBLP:conf/nips/AdebayoGMGHK18}
Julius Adebayo, Justin Gilmer, Michael Muelly, Ian~J. Goodfellow, Moritz Hardt,
  and Been Kim.
\newblock Sanity checks for saliency maps.
\newblock In {\em Advances in Neural Information Processing Systems 31}, pages
  9525--9536, 2018.

\bibitem{DBLP:conf/accv/AgarwalN20}
Chirag Agarwal and Anh Nguyen.
\newblock Explaining image classifiers by removing input features using
  generative models.
\newblock In Hiroshi Ishikawa, Cheng{-}Lin Liu, Tom{\'{a}}s Pajdla, and Jianbo
  Shi, editors, {\em Computer Vision - {ACCV} 2020 - 15th Asian Conference on
  Computer Vision, Kyoto, Japan, November 30 - December 4, 2020, Revised
  Selected Papers, Part {VI}}, volume 12627 of {\em Lecture Notes in Computer
  Science}, pages 101--118. Springer, 2020.

\bibitem{DBLP:conf/nips/Alvarez-MelisJ18}
David Alvarez{-}Melis and Tommi~S. Jaakkola.
\newblock Towards robust interpretability with self-explaining neural networks.
\newblock In {\em Advances in Neural Information Processing Systems 31: Annual
  Conference on Neural Information Processing Systems 2018, NeurIPS 2018,
  December 3-8, 2018, Montr{\'{e}}al, Canada}, pages 7786--7795, 2018.

\bibitem{DBLP:conf/fuzzIEEE/Arias-DuartPGG22}
Anna Arias{-}Duart, Ferran Par{\'{e}}s, Dario Garcia{-}Gasulla, and Victor
  Gimenez{-}Abalos.
\newblock Focus! rating {XAI} methods and finding biases.
\newblock In {\em {IEEE} International Conference on Fuzzy Systems, {FUZZ-IEEE}
  2022, Padua, Italy, July 18-23, 2022}, pages 1--8. {IEEE}, 2022.

\bibitem{DBLP:journals/inffus/ArrasOS22}
Leila Arras, Ahmed Osman, and Wojciech Samek.
\newblock {CLEVR-XAI:} {A} benchmark dataset for the ground truth evaluation of
  neural network explanations.
\newblock {\em Inf. Fusion}, 81:14--40, 2022.

\bibitem{bach2015pixel}
Sebastian Bach, Alexander Binder, Gr{\'e}goire Montavon, Frederick Klauschen,
  Klaus-Robert M{\"u}ller, and Wojciech Samek.
\newblock On pixel-wise explanations for non-linear classifier decisions by
  layer-wise relevance propagation.
\newblock {\em PLoS ONE}, 10(7):e0130140, 2015.

\bibitem{DBLP:journals/jmlr/BaehrensSHKHM10}
David Baehrens, Timon Schroeter, Stefan Harmeling, Motoaki Kawanabe, Katja
  Hansen, and Klaus{-}Robert M{\"{u}}ller.
\newblock How to explain individual classification decisions.
\newblock {\em Journal of Machine Learning Research}, 11:1803--1831, 2010.

\bibitem{DBLP:conf/icml/BalduzziFLLMM17}
David Balduzzi, Marcus Frean, Lennox Leary, J.~P. Lewis, Kurt~Wan{-}Duo Ma, and
  Brian McWilliams.
\newblock The shattered gradients problem: If resnets are the answer, then what
  is the question?
\newblock In {\em International Conference on Machine Learning (ICML)},
  volume~70 of {\em PMLR}, pages 342--350. {PMLR}, 2017.

\bibitem{DBLP:conf/ijcai/BhattWM20}
Umang Bhatt, Adrian Weller, and Jos{\'{e}} M.~F. Moura.
\newblock Evaluating and aggregating feature-based model explanations.
\newblock In {\em Proceedings of the Twenty-Ninth International Joint
  Conference on Artificial Intelligence, {IJCAI} 2020}, pages 3016--3022.
  ijcai.org, 2020.

\bibitem{DBLP:conf/icml/ChalasaniC00J20}
Prasad Chalasani, Jiefeng Chen, Amrita~Roy Chowdhury, Xi Wu, and Somesh Jha.
\newblock Concise explanations of neural networks using adversarial training.
\newblock In {\em Proceedings of the 37th International Conference on Machine
  Learning, {ICML} 2020, 13-18 July 2020, Virtual Event}, volume 119 of {\em
  Proceedings of Machine Learning Research}, pages 1383--1391. {PMLR}, 2020.

\bibitem{he2015delving}
Kaiming He, Xiangyu Zhang, Shaoqing Ren, and Jian Sun.
\newblock Delving deep into rectifiers: Surpassing human-level performance on
  imagenet classification.
\newblock In {\em Proceedings of the IEEE international conference on computer
  vision}, pages 1026--1034, 2015.

\bibitem{Resnet:he2016deep}
Kaiming He, Xiangyu Zhang, Shaoqing Ren, and Jian Sun.
\newblock Deep residual learning for image recognition.
\newblock In {\em {IEEE} Conference on Computer Vision and Pattern Recognition
  (CVPR)}, pages 770--778, 2016.

\bibitem{10.1093/biomet/56.3.635}
D.~V. Hinkley.
\newblock {On the ratio of two correlated normal random variables}.
\newblock {\em Biometrika}, 56(3):635--639, 12 1969.

\bibitem{DBLP:journals/csr/HuangKRSSTWY20}
Xiaowei Huang, Daniel Kroening, Wenjie Ruan, James Sharp, Youcheng Sun, Emese
  Thamo, Min Wu, and Xinping Yi.
\newblock A survey of safety and trustworthiness of deep neural networks:
  Verification, testing, adversarial attack and defence, and interpretability.
\newblock {\em Comput. Sci. Rev.}, 37:100270, 2020.

\bibitem{Kindermans2019}
Pieter-Jan Kindermans, Sara Hooker, Julius Adebayo, Maximilian Alber,
  Kristof~T. Sch{\"u}tt, Sven D{\"a}hne, Dumitru Erhan, and Been Kim.
\newblock {\em The (Un)reliability of Saliency Methods}, pages 267--280.
\newblock Springer International Publishing, Cham, 2019.

\bibitem{DBLP:journals/corr/abs-2106-07475}
Narine Kokhlikyan, Vivek Miglani, Bilal Alsallakh, Miguel Martin, and Orion
  Reblitz{-}Richardson.
\newblock Investigating sanity checks for saliency maps with image and text
  classification.
\newblock {\em arXiv preprint arXiv:2106.07475}, 2021.

\bibitem{Lapuschkin2019}
Sebastian Lapuschkin, Stephan W{\"a}ldchen, Alexander Binder, Gr{\'e}goire
  Montavon, Wojciech Samek, and Klaus-Robert M{\"u}ller.
\newblock Unmasking clever hans predictors and assessing what machines really
  learn.
\newblock {\em Nature Communications}, 10(1):1096, Mar 2019.

\bibitem{DBLP:journals/corr/abs-2103-10689}
Xuhong Li, Haoyi Xiong, Xingjian Li, Xuanyu Wu, Xiao Zhang, Ji Liu, Jiang Bian,
  and Dejing Dou.
\newblock Interpretable deep learning: Interpretations, interpretability,
  trustworthiness, and beyond.
\newblock {\em arXiv preprint arXiv:2103.10689}, 2021.

\bibitem{lukemelaseffnet}
Luke Melas.
\newblock 2019.

\bibitem{Montavon2019gradvspropagationbasedcomparison}
Gr{\'e}goire Montavon.
\newblock {\em Gradient-Based Vs. Propagation-Based Explanations: An Axiomatic
  Comparison}, pages 253--265.
\newblock Springer International Publishing, Cham, 2019.

\bibitem{DBLP:journals/dsp/MontavonSM18}
Gr{\'{e}}goire Montavon, Wojciech Samek, and Klaus{-}Robert M{\"{u}}ller.
\newblock Methods for interpreting and understanding deep neural networks.
\newblock {\em Digit. Signal Process.}, 73:1--15, 2018.

\bibitem{DBLP:journals/corr/abs-2007-07584}
An{-}phi Nguyen and Mar{\'{\i}}a~Rodr{\'{\i}}guez Mart{\'{\i}}nez.
\newblock On quantitative aspects of model interpretability.
\newblock {\em arXiv preprint arXiv:2007.07584}, 2020.

\bibitem{nilsson2020understanding}
Jim Nilsson and Tomas Akenine-Möller.
\newblock Understanding ssim.
\newblock {\em arXiv preprint arXiv:2006.13846}, 2020.

\bibitem{NEURIPS2019_9015}
Adam Paszke, Sam Gross, Francisco Massa, Adam Lerer, James Bradbury, Gregory
  Chanan, Trevor Killeen, Zeming Lin, Natalia Gimelshein, et~al.
\newblock Pytorch: An imperative style, high-performance deep learning library.
\newblock In {\em Advances in Neural Information Processing Systems 32}, pages
  8024--8035. Curran Associates, Inc., 2019.

\bibitem{DBLP:journals/tnn/SamekBMLM17}
Wojciech Samek, Alexander Binder, Gr{\'{e}}goire Montavon, Sebastian
  Lapuschkin, and Klaus-Robert M{\"{u}}ller.
\newblock Evaluating the visualization of what a deep neural network has
  learned.
\newblock {\em {IEEE} Transactions on Neural Networks and Learning Systems},
  28(11):2660--2673, 2017.

\bibitem{DBLP:journals/pieee/SamekMLAM21}
Wojciech Samek, Gr{\'{e}}goire Montavon, Sebastian Lapuschkin, Christopher~J.
  Anders, and Klaus{-}Robert M{\"{u}}ller.
\newblock Explaining deep neural networks and beyond: {A} review of methods and
  applications.
\newblock {\em Proc. {IEEE}}, 109(3):247--278, 2021.

\bibitem{simonyan2013deep}
Karen Simonyan, Andrea Vedaldi, and Andrew Zisserman.
\newblock Deep inside convolutional networks: Visualising image classification
  models and saliency maps.
\newblock {\em arXiv preprint arXiv:1312.6034}, 2013.

\bibitem{DBLP:conf/icml/SixtGL20}
Leon Sixt, Maximilian Granz, and Tim Landgraf.
\newblock When explanations lie: Why many modified {BP} attributions fail.
\newblock In {\em International Conference on Machine Learning ({ICML})},
  volume 119 of {\em PMLR}, pages 9046--9057. {PMLR}, 2020.

\bibitem{DBLP:journals/corr/SmilkovTKVW17}
Daniel Smilkov, Nikhil Thorat, Been Kim, Fernanda~B. Vi{\'{e}}gas, and Martin
  Wattenberg.
\newblock Smoothgrad: removing noise by adding noise.
\newblock {\em arXiv preprint arXiv:1706.03825}, 2017.

\bibitem{GuidebackPropagation:springenberg2014striving}
J Springenberg, Alexey Dosovitskiy, Thomas Brox, and M Riedmiller.
\newblock Striving for simplicity: The all convolutional net.
\newblock In {\em ICLR (workshop track)}, 2015.

\bibitem{Stuart1991}
A. Stuart and K. Ord.
\newblock {\em {Kendall's advanced theory of statistics}}, volume 2, Classical
  Inference and Relationship.
\newblock Wiley, fifth edition, March 1991.

\bibitem{DBLP:journals/corr/abs-1806-04205}
Mukund Sundararajan and Ankur Taly.
\newblock A note about: Local explanation methods for deep neural networks lack
  sensitivity to parameter values.
\newblock {\em arXiv preprint arXiv:1806.04205}, 2018.

\bibitem{DBLP:journals/corr/SundararajanTY17}
Mukund Sundararajan, Ankur Taly, and Qiqi Yan.
\newblock Axiomatic attribution for deep networks.
\newblock In {\em {ICML}}, volume~70 of {\em Proceedings of Machine Learning
  Research}, pages 3319--3328. {PMLR}, 2017.

\bibitem{wang2004image}
Zhou Wang, A.C. Bovik, H.R. Sheikh, and E.P. Simoncelli.
\newblock Image quality assessment: from error visibility to structural
  similarity.
\newblock {\em IEEE Transactions on Image Processing}, 13(4):600--612, 2004.

\bibitem{DBLP:journals/corr/abs-2110-14297}
Gal Yona and Daniel Greenfeld.
\newblock Revisiting sanity checks for saliency maps.
\newblock {\em arXiv preprint arXiv:2110.14297}, 2021.

\bibitem{DBLP:journals/ijcv/ZhangBLBSS18}
Jianming Zhang, Sarah~Adel Bargal, Zhe Lin, Jonathan Brandt, Xiaohui Shen, and
  Stan Sclaroff.
\newblock Top-down neural attention by excitation backprop.
\newblock {\em Int. J. Comput. Vis.}, 126(10):1084--1102, 2018.

\end{thebibliography}
\bibliographystyle{ieee_fullname}

\newpage
\appendix
\onecolumn
\renewcommand\thefigure{App.\arabic{figure}}
\setcounter{figure}{0}
\begin{center}
    
\noindent\Large{\textbf{Appendix}}
\end{center}

\section{Details on Experiments}
\label{ssec:experiment_details}

\subsection{Randomization-based Sanity Checks vs. Faithfulness}
\label{ssec:experiment_details:faithfulness_comp}
The ResNet-50 and DenseNet-121 are used as provided by the Torchvision package of PyTorch \cite{NEURIPS2019_9015}. For the EfficientNet-B0 we resort to a pretrained model provided by the github of Luke Melas \cite{lukemelaseffnet}. All results are averaged over the first 1000 images from the ImageNet validation set. 

For model-randomization-based sanity checks testing we reset the layers as per the initialization introduced by \cite{he2015delving}. 
We report model randomization for a partial set of layers, as these results are in already known from \cite{DBLP:conf/nips/AdebayoGMGHK18}, which makes an exhaustive computation for each layer unnecessary. For the ResNet-50 we randomize the fully connected layer, and in each step we randomize all layers from the last randomized layer until the next layer with name {\tt .conv1} as per Torchvision until we have randomized 16 {\tt .conv1}-named layers. For the DenseNet-121 we also randomize the fully connected layer, and in each step we randomize all layers from the last randomized layer until the next third layer with name {\tt .conv1} until we have randomized 63 {\tt .conv1}-named layers. For the EfficientNet-B0 we randomize the fully connected layer, and in each step we randomize all layers from the last randomized layer until the next layer with name {\texttt{.\_depthwise\_conv}} until we have randomized 17 {\tt .\_depthwise\_conv}-named layers. 

For the perturbation-based testing we create a blurred version of the original image, using a constant blur kernel of kernel size 15. We perform the perturbation by replacing a region of kernel size 8 or 15 in the original image by a patch from the blurred version. We do this for the 30 regions in an image which have the highest average attribution map score. Unlike the random draw for a patch used \cite{DBLP:journals/tnn/SamekBMLM17}, using a blurred copy results in a less pronounced outlier structure due to preservation of color statistics while removing texture. We measure the decrease of the prediction function under the iterative replacement of the highest scoring patches of the image by the corresponding patches from the blurred copy.

\subsection{Forward Pass-Adaptive $\beta$-rule}

It is used in the experiments for model faithfulness estimation. The idea is based on the interpretation that $\frac{\beta}{1+\beta}$ in LRP-$\beta$ is the fraction of redistributed negative to positive relevance. An adaptive way to determine its value can be derived by setting it equal to the corresponding fraction $\frac{ -\sum_i  (w_ix_i)_- }{ \sum_i(w_ix_i)_+} $ of the input statistics of a neuron, and solving it for $\beta$ as in:
\begin{align}
\frac{\beta}{1+\beta} &= \frac{ -\sum_i  (w_ix_i)_- }{ \sum_i(w_ix_i)_+} \\
\Rightarrow \beta &= \frac{   -\sum_i (w_ix_i)_- }{ \sum_i(w_ix_i)_+ - \sum_i (w_ix_i)_- }
\end{align}
We use a value of $\beta_*=\min(\beta,3.0)$ in all experiments.

\subsection{Additional Results of Model Faithfulness Experiments}
\label{ssec:morefaithfulessexperiments}
Please see Figure \ref{fig:perturb_gbpvsgrad_k8} for results with a kernel size of $8$.
\begin{figure*}[!bht]
\centering     
\subcaptionbox{ResNet50}{\label{fig:perturb_gbpvsgrad_k8:a}
\includegraphics[width=0.49\linewidth]{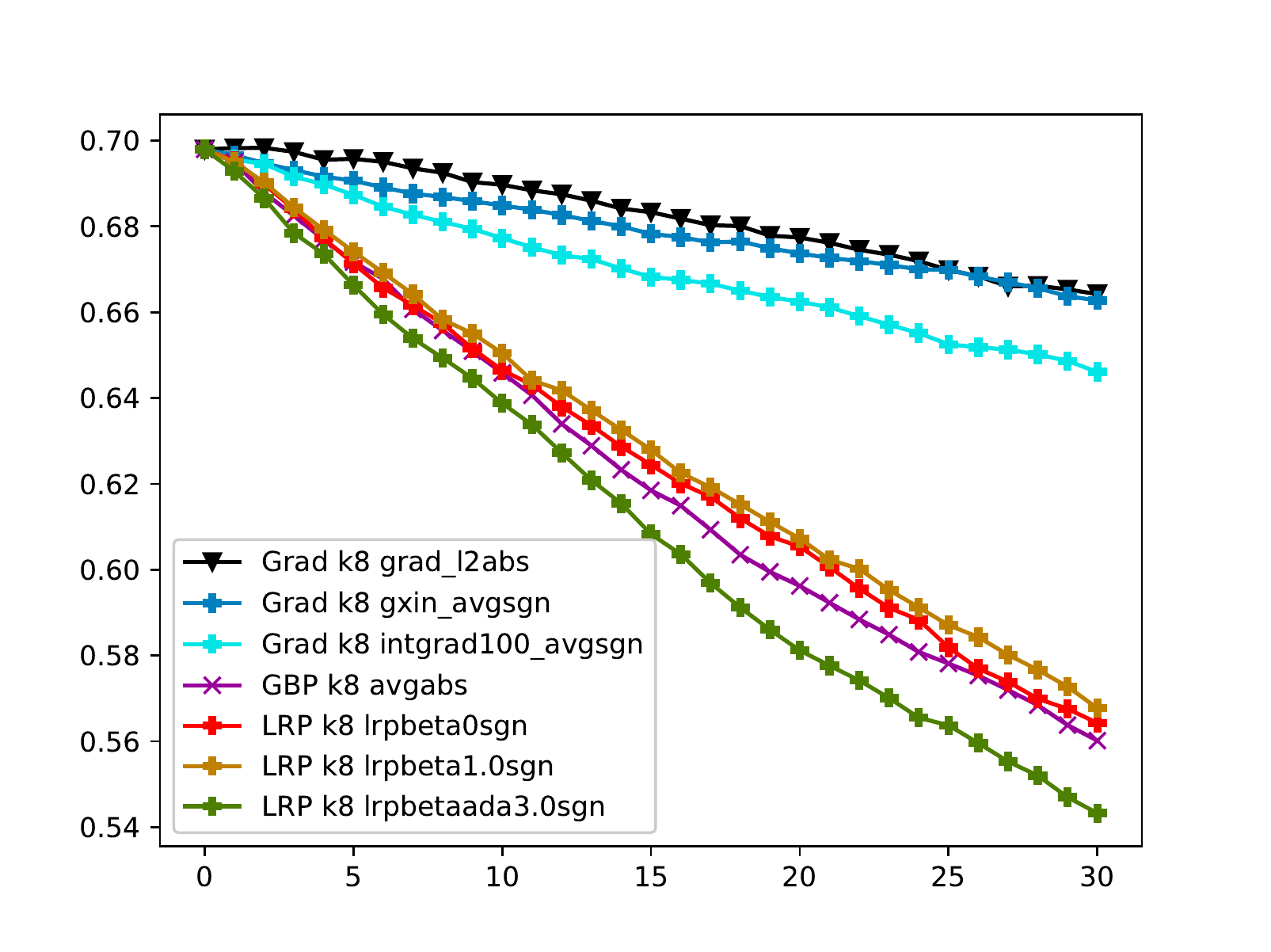}
}
\subcaptionbox{DenseNet121}{\label{fig:perturb_gbpvsgrad_k8:b}
\includegraphics[width=0.49\linewidth]{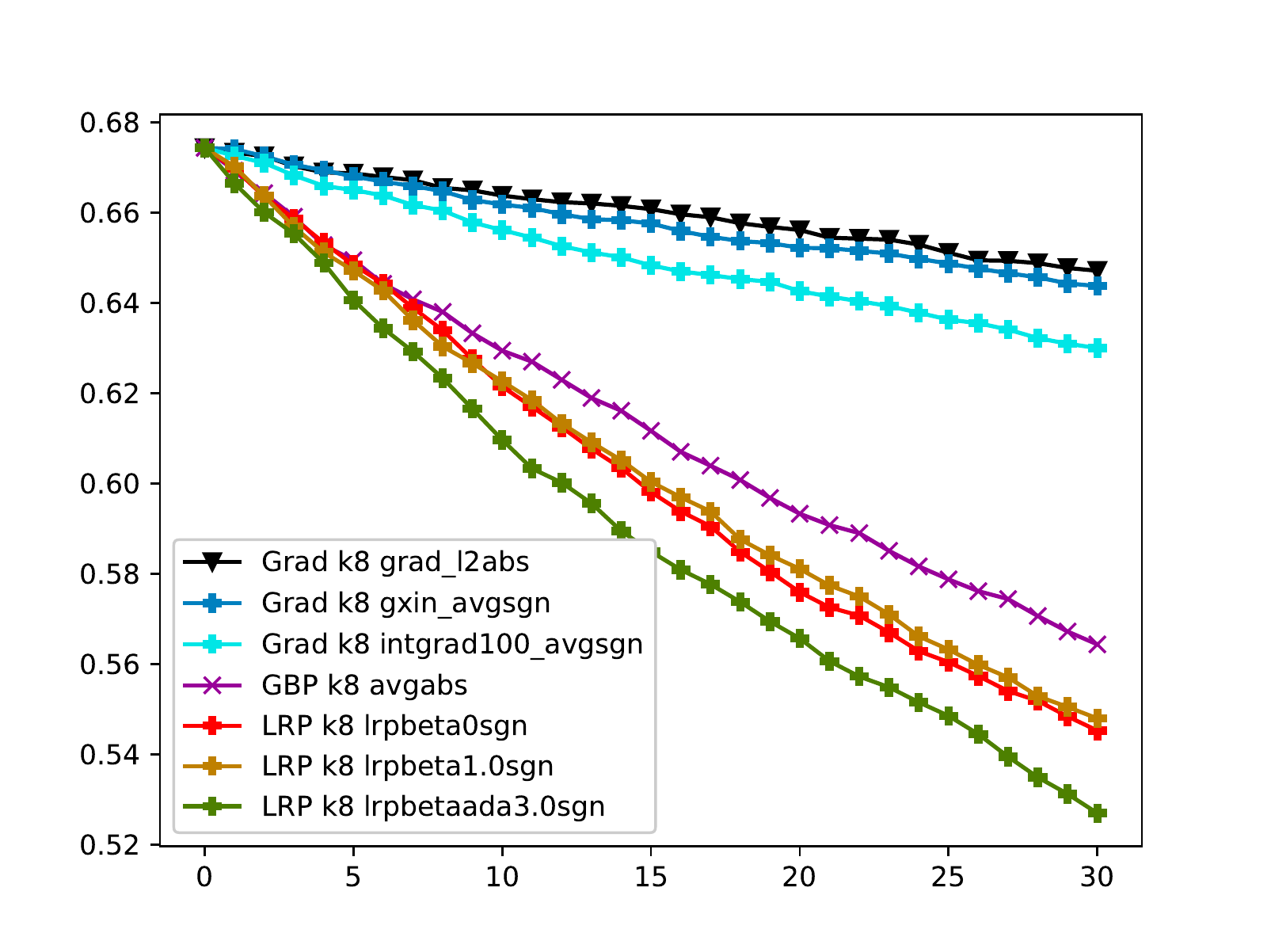}
}
\subcaptionbox{EfficientNet-B0}{\label{fig:perturb_gbpvsgrad_k8:c}
\includegraphics[width=0.49\linewidth]{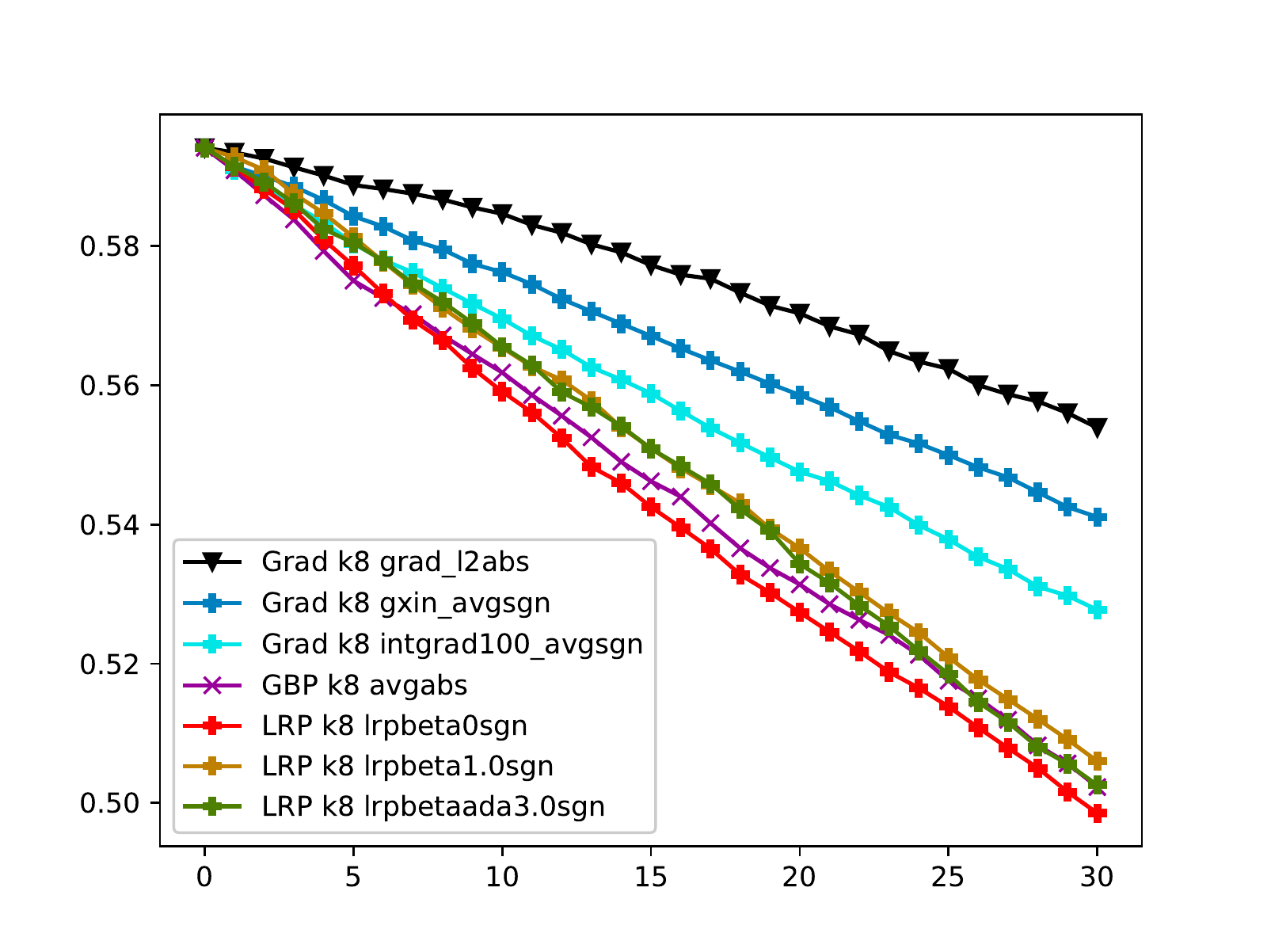}
}
\caption{\label{fig:perturb_gbpvsgrad_k8} Results of model faithfulness testing by measuring the correlation to iterative occlusion with a kernel size of 8. The comparison shows the gradient, gradient $\times$ input, integrated gradient, guided backpropagation and several LRP approaches. The occlusion is performed by taking patches from a blurred copy of the original image. The figure shows the softmax scores. \textit{Lower is better.}}
\end{figure*}

\section{Proof of Theorem 1}
\label{ssec:app:proofofthm1}
\begin{proof}
Consider the term in Equation (1) of the main paper. 
Since we consider processes with $\sigma_{AB} \ge 0$, this attains the minimum at $\sigma_{AB}=0$, resulting in
\begin{align}
 &\min_{  \sigma_{AB} \ge 0 } \left|\frac{2\mu_A \mu_B +C_1}{ \mu^2_A + \mu^2_B +C_1  } \frac{2\sigma_{AB} +C_2}{ \sigma^2_A + \sigma^2_B +C_2} \right|\\
 & =  \frac{|2\mu_A \mu_B +C_1|}{ \mu^2_A + \mu^2_B +C_1  } \cdot \min_{  \sigma_{AB} \ge 0 }  \frac{|2\sigma_{AB}+ C_2|}{ \sigma^2_A + \sigma^2_B +C_2} \\
 & =  \frac{|2\mu_A \mu_B +C_1|}{ \mu^2_A + \mu^2_B +C_1  }\frac{C_2}{ \sigma^2_A + \sigma^2_B +C_2}\\
  & \le  \frac{2|\mu_A \mu_B| +C_1}{ \mu^2_A + \mu^2_B +C_1  }\frac{C_2}{ \sigma^2_A + \sigma^2_B +C_2}\\
 &\le \frac{\mu^2_A + \mu^2_B +C_1}{ \mu^2_A + \mu^2_B +C_1  }\frac{C_2}{ \sigma^2_A + \sigma^2_B +C_2}
 \end{align}
The last inequality holds due to $ \pm 2ab \le a^2  +b^2   $.

\end{proof}

\section{The Sensitivity of Spearman Rank Correlation Minimization Towards Noise}
\label{ssec:spearmansens}
In Section 3 of the main paper we demonstrated the sensitivity of the SSIM metric towards random attributions. The same (in terms of ranks) holds for the other distance metric employed by \cite{DBLP:conf/nips/AdebayoGMGHK18}, the Spearman Rank Correlation, given as 

\begin{equation}
\label{eq:srankeq}
    \frac{\sigma_{R(A)R(B)}}{\sigma_{R(A)}\sigma_{R(B)}}~,
\end{equation}

with $R(A)$ and $R(B)$ being the ranks derived from attribution maps $A$ and $B$. 
The following Theorem and Proof show the sensitivity of this metric's minimization towards random noise analogously to Theorem 1 and the corresponding Proof:

\begin{theorem}
\label{Theor.3}
Consider the set of all statistical processes with non-negative expected covariance between the corresponding ranks $\sigma_{R(A)R(B)} \ge 0$. 

Then the expected Spearman Rank Correlation is minimized by a statistical process with zero covariance between the corresponding ranks. 

\end{theorem}
\begin{proof}
Consider the term in Equation \eqref{eq:srankeq}. Since $\sigma_{R(A)R(B)} \ge 0$, and the standard deviations $\sigma_{R(A)}, \sigma_{R(B)} \ge 0$, this attains the minimum at $\sigma_{R(A)R(B)}=0$, resulting in
\begin{align}
 &\min_{  \sigma_{R(A)R(B)} \ge 0 } \left( \frac{\sigma_{R(A)R(B)}}{\sigma_{R(A)}\sigma_{R(B)}} \right)\\
 & =  \frac{0}{\sigma_{R(A)}\sigma_{R(B)}}= 0
 \end{align}
\end{proof}

Of course, the ranks and their covariance depend not only on the attribution maps, but also on the employed ranking function. However, if simply the sorted indices of attribution maps (or their absolute values) are used as ranks, then $\sigma_{R(A)R(B)} \ge 0$ iff $\sigma_{AB} \ge 0$, and Theorem \ref{Theor.3} holds for all statistical processes with non-negative expected covariance $\sigma_{AB} \ge 0$.

\section{The Sensitivity of Normalized MSE Maximization Towards Noise}
\label{sec:msesens}
One may consider to replace the \gls{ssim} by a \gls{mse}. This comes with another topic to be considered: Different methods to compute attribution maps may exhibit different patch-wise variances, which will affect the scale of differences used in model-randomization-type sanity checks unrelated to the effects coming from the model randomization itself. This raises the question of how to normalize attribution maps in order to ensure a comparability of the distances computed using different attribution methods. 

We consider for the case of \gls{mse} attribution maps which are normalized by dividing the attribution map by the square-root of its average second moment estimate:
\begin{equation}
\label{eq:msenorm}
    \frac{A}{ \left(\hat{E}[ A^2 ]\right)^{1/2} } =  \frac{A_{h,w}}{  \left(\frac{1}{HW}\sum_{h',w'}h_{h',w'}^2\right)^{1/2} }~,
\end{equation} 
where $A_{h,w}$ is the value of the attribution map at pixel location $(h, w)$ and $H, W$ denote the attribution map height and width, respectively.  
This ensures that the average squared distance of an attribution score per pixel from the attribution value of $0$ is 1. One may ask why we did not choose the more common standard deviation \footnote{$\left(\hat{E}[ A_{h,w}^2 ]-(\hat{E}[ A_{h,w} ])^2\right)^{1/2}$ for reference}. Standard deviation normalizes the average squared distance of a pixel-score from the mean of a patch to be one. However, the mean score over the pixels of the different attribution methods (such as gradient with $\ell_2$-norm over the RGB-subpixels, also known as Sensitivity~\cite{DBLP:journals/jmlr/BaehrensSHKHM10,simonyan2013deep}, gradient with averaging over the RGB-subpixels, gradient $\times$ input and integrated gradients) has no particular meaning for explaining the prediction in the context of the above methods. The value of zero ($0$) has for all of above methods the meaning of being the estimate of non-contribution to the prediction, which justifies the choice of second moment estimates. Thus we ensure an equal average distance from the point of no-contribution by this type of normalization.
This ensures better comparability of distances among attribution maps computed for different attribution map processes.

Using the \gls{mse} does not resolve the issue that a zero covariance attribution map yields the best results among all statistical processes with non-negative covariances $\sigma_{AB}\ge 0$. In the following $A$, $B$ can be single subpixels. It directly translates to patches when using $E[\|A-B\|^2_2]$ instead.

The following theorem is again meant to be used with two different attribution maps $A$, $B$, e.g.,~coming from a model and a partially randomized variant of it, over the same patch location.
\begin{theorem}
\label{Theor.2}
Consider the set of all statistical processes 
with non-negative expected covariance for each patch $\sigma_{AB} \ge 0$. Then the expected \gls{mse} can be maximized by using a statistical process with zero covariance and the maximal value is 
 $2-2\frac{\mu_A\mu_B}{E[A^2]^{1/2}E[B^2]^{1/2}}$
\end{theorem}

\begin{proof}
\begin{align}
&E\left[\left( \frac{A}{ E[A^2]^{1/2}  } - \frac{B}{ E[B^2]^{1/2}  }\right)^2\right]\\
&= \frac{ E[A^2]}{E[A^2]} -2 \frac{E[AB]}{E[A^2]^{1/2}E[B^2]^{1/2}} + \frac{ E[B^2]}{E[B^2]}  \\
& =  2 - 2 \frac{\sigma_{AB}}{E[A^2]^{1/2}E[B^2]^{1/2}}-2\frac{\mu_A\mu_B}{E[A^2]^{1/2}E[B^2]^{1/2}} \\
& \le 2 -2\frac{\mu_A\mu_B}{E[A^2]^{1/2}E[B^2]^{1/2}} \label{eq:normalizedmse}
\end{align}
\end{proof}

With respect to the influence of means, for $\mu_A\mu_B \approx 0$, this would result in a MSE of $2$. Note that we can observe from Figure \ref{fig:mse} that the \gls{mse} indeed attains a value close to $2.0$ for certain methods which perform well in model-randomization-type sanity checks, such as gradient and integrated gradient. In light of above theorem, this finding is conspicuous as it may indicate a correlation $\sigma_{AB}$ and means $\mu_A$, $\mu_B$ close to zero, in the sense of high gradient shattering noise.

This result shows, that the contribution from the randomization of a model and the noise from the attribution map process are still entangled when using \gls{mse} for model-randomization-type sanity checks. Consequently, using attribution map processes with a lower degree of correlation within a patch makes them appear more favorable when using model-randomization-type sanity checks to compare attribution maps. As noted before, a lower degree of correlation may originate from using a process with a high amount of zero-correlation noise, and in the worst case, from a statistically independent random process.  

On a side note, using an unnormalized MSE would result in
\begin{align}
E\left[\left( A - B\right)^2\right] & = E[A^2] -2E[AB] + E[B^2]\\
& = \sigma_A^2   -2\sigma_{AB}  + \sigma_B^2 \nonumber\\ &+E[A]^2 -2E[A]E[B]+ E[B]^2\\
& = \sigma_A^2  + \sigma_B^2  + (\mu_A-\mu_B)^2 -2\sigma_{AB}\\
& \le \sigma_A^2  + \sigma_B^2 + (\mu_A-\mu_B)^2
\end{align}
This would again be a measure that is maximized among the set of processes with non-negative covariance by using $\sigma_{AB}=0$ and additionally be sensitive to increasing patch-wise variances $\sigma_A^2$, $\sigma_B^2$ in processes. The proposed normalization by the second moment puts a bound on the sensitivity to patch-wise variances.

In summary, this section shows that replacing minimization of a similarity by maximization of a well-known squared distance still retains a sensitivity and possible preference towards attribution map methods with low correlation when viewed as a statistical process.

\subsection{A Note on Normalization}
The scores obtained from different attribution methods generally do not share the same range of values. Therefore, in order to compare them, some sort of normalization is required, however, care has to be taken in doing so as to not alter or destroy any information provided by attributions. The seminal work on model-randomization-based sanity checks \cite{DBLP:conf/nips/AdebayoGMGHK18} uses normalization by division with statistics using the maximal absolute value of an attribution map. This, however, may introduce additional variance in the measurement when computing differences of attribution maps: 

Minima or maxima are the only statistics among quantile estimators which do not converge for an increasing sample size to a finite expectation. One can easily see this by considering random draws from a normal distribution. The maximum will tend to infinity as the sample size $n\rightarrow \infty$ increases. 

More formally, as noted in \cite{Stuart1991}, the distribution of several known quantile estimators for the p-th quantile of a distribution is approximately normal with a variance of 
\begin{equation}
    \sigma^2 = \frac{  p(1-p)}{n f(F^{-1}(p))}
\end{equation}
 where $f(\cdot)$ is the density, $F(\cdot)$ the cumulative density of the distribution which is used to draw samples used to compute the p-th quantile estimator, and $n$ is the sample size. Thus for quantile estimators $p \approx 0$, $p \approx 1$ with values $F^{-1}(p)$ at the tails of the distribution, where the value of the density $f(\cdot)$ is low, the variance $\sigma^2$ will become unbounded, as long as $f(F^{-1}(p))$ decays faster than $\mathcal{O}(p^{-1})$ or $\mathcal{O}( (1-p)^{-1})$, which is the case for a higher degree polynomial or exponential decay. 

It should be noted that normalization aiming at a proper perception by the human eye and normalization for the sake of comparability of distances are non-equivalent goals. The former needs to ensure a bounded range, and color intensities which are well perceivable.

Normalization by the maximum yields a high variance of the estimator, and, while suitable for visualization to the human eye, does not preserve a quantity useful for the comparison of distances across different models under parameter randomization. For this reason we will consider a different normalization as outlined above.


\section{Proof of Theorem 2}
\label{ssec:app:proofofthm2}
\begin{proof}
To see this, consider two sets of non-negative input activations for a neuron, $X_{L}$ and $X_{S}$. We assume that each input from $X_{L}$ is by a factor of $K$ larger than each input from $X_{S}$ such that: 
\begin{align}
\min_{x_l \in X_L} x_l \ge K \max_{x_s \in X_S} x_s \ . \label{eq:thm2requirement} 
\end{align}

In order for a single $x_s$ to have at least the same effect on the output as a single $x_l>0$, it requires $w_s x_s \ge w_l x_l$ and thus for the weights $w_s \ge Kw_l$. This corresponds to a ratio distribution of two zero-mean normal variables, which is known to have a Cauchy density, as for example shown in \cite{10.1093/biomet/56.3.635}
\begin{align}
f\left(\frac{w_s}{w_l} =K\right) = \frac{1}{\pi \gamma} \frac{1}{K^2/\gamma^2+1}, \ \gamma = \frac{\sigma_s}{\sigma_l}  \ . 
\end{align}
The quantity of interest in this case is the tail-CDF 
\begin{align}
P\left(\frac{w_s}{w_l} \ge K\right)=1-CDF_{\gamma}(K)\ .
\end{align}

In order for each of the neurons in the small-value set to have the {same summed contribution} to the output, we require
\begin{align}
\sum_{ x_s \in X_S} w_s x_s \ge \sum_{ x_l \in X_L} w_l x_l~. \label{eq:xsgexl}
\end{align}

This can be combined together as follows. 
\begin{align}
\sum_{ x_s \in X_S} w_s x_s \text{ and } \sum_{ x_l \in X_L} w_l x_l
\end{align}
are normally distributed random variables with respect to draws of the weights $w$ with zero mean and variances
\begin{align}
\sigma_S= \sum_{ x_s \in X_S} x_s^2 \text{ and } \sigma_L=\sum_{ x_l \in X_L} x_l^2 ~.
\end{align}

Thus, for $\sum_{ x_l \in X_L} w_l x_l>0$ the requirement in Equation \eqref{eq:xsgexl} translates into the probability of the ratio 
\begin{align}
\frac{\sum_{ x_s \in X_S} w_s x_s}{\sum_{ x_l \in X_L} w_l x_l} \ge 1
\end{align}
This is the cumulative tail probability $P(Z \ge 1) = 1 -CDF_{\gamma}(1)$ with a parameter $\gamma_1$ given as
\begin{align}
\gamma_1 = \sqrt{\frac{\sigma_S}{\sigma_L}} =  \sqrt{\frac{\sum_{ x_s \in X_S} x_s^2}{\sum_{ x_l \in X_L} x_l^2}}
\end{align}
The Cauchy distribution obtains larger cumulative tail probabilities for larger values of the parameter $\gamma$. Therefore for an upper bound on cumulative tail probabilities, we need to obtain an upper bound on $\gamma_1$.
\begin{align}
\gamma_1 &=   \sqrt{\frac{\sum_{ x_s \in X_S} x_s^2}{\sum_{ x_l \in X_L} x_l^2}} \label{eq:gamma1} \\
&\le  \sqrt{\frac{\sum_{ x_s \in X_S} \max_{x_s \in X_S} x_s^2}{\sum_{ x_l \in X_L} \min_{x_l \in X_L} x_l^2 }} \\
&= \sqrt{\frac{|X_S| \max_{x_s \in X_S} x_s^2}{|X_L| \min_{x_l \in X_L} x_l^2 }}\\
& \stackrel{\text{Eq.\eqref{eq:thm2requirement}}}{\le} \sqrt{\frac{|X_S| \frac{1}{K^2}\min_{x_l \in X_L} x_l^2}{|X_L| \min_{x_l \in X_L} x_l^2 }} \\
&= \sqrt{\frac{|X_S|}{|X_L|}}\frac{1}{K}
\end{align}
where we used Equation \eqref{eq:thm2requirement} to get a term depending on $K$.
Plugging in this upper bound $\gamma_1$ into the CDF shows
\begin{align}
CDF_{\gamma_1}(1) &= 0.5+\frac{1}{\pi} \arctan ( \frac{1-0}{\gamma_1} ) \\
&= 0.5+\frac{1}{\pi} \arctan \left( \frac{K}{\sqrt{\frac{|X_S|}{|X_L|}}} \right)\\
& = CDF_{\gamma_2}(K), \ \gamma_2=  \sqrt{\frac{|X_S|}{|X_L|}}~.
\end{align}

Therefore we obtain the cumulative tail CDF of a Cauchy distribution from the value of $K$ onwards $P(Z \ge K)$  with a parameter $\gamma_2=  \sqrt{\frac{|X_S|}{|X_L|}}$.
\end{proof}

Section \ref{ssec:probforwardpassovertaking} in this supplement provides estimates for this probability for three trained deep neural networks which provides empirical evidence for the sparsity.\\

If one would consider average contributions 
\begin{align}
\frac{1}{|X_S|}\sum_{ x_s \in X_S} w_s x_s \ge \frac{1}{|X_L|}\sum_{ x_l \in X_L} w_l x_l~, 
\end{align}
then one would obtain the analogous result with an inverted parameter $\gamma_{2,avg} =\sqrt{\frac{|X_L|}{|X_S|}}$.

A reason to consider such averages instead of sums would be the case when one is interested to analyze when two regions of an input would achieve the same average explanation score per input element of the respective regions. This case corresponds in an attribution map to two regions with the same average color intensity per pixel.

This can be shown as follows. If we consider 
\begin{align}
\frac{1}{|X_S|}\sum_{ x_s \in X_S} w_s x_s \text{ and } \frac{1}{|X_L|}\sum_{ x_l \in X_L} w_l x_l~,
\end{align}
then these are normally distributed random variables with respect to draws of the weights $w$ with zero mean and variances
\begin{align}
\sigma_S= \frac{1}{|X_S|^2}\sum_{ x_s \in X_S} x_s^2 \text{ and } \sigma_L=\frac{1}{|X_L|^2}\sum_{ x_l \in X_L} x_l^2 ~.
\end{align}
The difference to the proof above is a multiplicative factor in $\gamma_1$ in Equation \eqref{eq:gamma1} of
\begin{align}
\sqrt{\frac{\frac{1}{|X_S|^2}}{\frac{1}{|X_L|^2}}} &= \frac{|X_L|}{|X_S|} \\
\Rightarrow \gamma_{2,avg} &= \frac{|X_L|}{|X_S|} \gamma_2 = \frac{|X_L|}{|X_S|} \sqrt{\frac{|X_S|}{|X_L|}}  = \sqrt{\frac{|X_L|}{|X_S|}}
\end{align}

\section{The Monotonicity Property of selected Explanation Methods}
\label{ssec:monotonicityofrelevance}

We show here that several explanation methods satisfy the positive monotonicity property that if we consider two inputs $x_i$, $x_j$ which have no other connections except to neuron $y$, then $w_ix_i \ge w_j x_j>0$ implies $|R(x_i)| \ge |R(x_j)|$ .

\subsection{Positive Monotonicity for Gradient $\times$ Input}

\begin{align}
z &  = g( \sum_k w_k x_k +b )\\
    R(x_i) &= \frac{\partial f }{\partial z} \frac{\partial z }{\partial x_i}(x) x_i =   \frac{\partial f }{\partial z} g'(\cdots) w_i x_i\\
 &   \sum_k w_k x_k +b >0,  w_ix_i > w_j x_j>0 \Rightarrow \\
 | R(x_i)| &= \left|\frac{\partial f }{\partial z}\right| |g'(\cdots)| | w_i x_i| \\
 > |R(x_j)| & = \left|\frac{\partial f }{\partial z}\right| |g'(\cdots)| | w_j x_j|
\end{align}
In fact, a stronger version holds here:  $|w_ix_i| \ge | w_j x_j| $ implies $|R(x_i)| \ge |R(x_j)|$

\subsection{Positive Monotonicity for Shapley Values}

This holds when $w_i x_i > w_j x_j >0$ and the activation function $g$ is monotonously non-decreasing. In that case, for all subsets $S: i \notin S, j\notin S$:
\begin{align}
f(S \cup \{i\} )&= g (\sum_{k \in S } w_k x_k +b  + w_i x_i) \\
& \ge  g (\sum_{k \in S } w_k x_k +b  + w_j x_j) = f(S \cup \{j\} )\\
\Rightarrow \phi(i) & =  \sum_S c_{|S|} (f(S \cup \{i\} ) - f(S)) \\
&\ge \sum_S c_{|S|} (f(S \cup \{j\} ) - f(S))  =  \phi(j)\ ,
\end{align}
where 
\begin{align}
c_{|S|} = \frac{1}{d \binom{d-1}{|S|}}
\end{align}
are the normalizing constants used in the exact computation of Shapley values.

\subsection{Positive Monotonicity for LRP-$\beta$}

\begin{align}
R(i) &= R(z) (1+\beta) \frac{ (w_i x_i)_{+} }{\sum_k (w_k x_k)_{+}} \nonumber\\
&- R(z) \beta \frac{(w_i x_i)_{-} }{ \sum_k (w_k x_k)_{-} }\\
w_ix_i>0 & \Rightarrow R(i) = R(z) (1+\beta) \frac{ (w_i x_i)_{+} }{\sum_k (w_k x_k)_{+}}  \\
 w_j x_j>0 & \Rightarrow R(j) = R(z) (1+\beta) \frac{ (w_j x_j)_{+} }{\sum_k (w_k x_k)_{+}}  \\
w_ix_i \ge w_j x_j>0 & \Rightarrow  (w_ix_i)_+ \ge (w_j x_j)_+\\
\Rightarrow |R(i)| &= |R(z)| (1+\beta) \frac{ (w_i x_i)_{+} }{\sum_k (w_k x_k)_{+}} \\
&\ge |R(j)|
\end{align}

In fact, a stronger version holds here:  $|w_ix_i| \ge | w_j x_j| $ and $sign(w_ix_i)= sign(w_jx_j)$ implies $|R(x_i)| \ge |R(x_j)|$ .

\section{Positive Explanation Score Dominance in ReLU Networks with Positive Logits}
\label{ssec:posdominance}

In this section we briefly show another property to hold, when explaining positive logits in ReLU networks with non-positive biases, irrespective of the randomization. 

The property is that the positive evidence will dominate the negative evidence in every layer until the input, \emph{under the condition that the explanation is additive for ReLU units with positive outputs.} An exception to it would occur when one has large positive biases, and one would attribute explanation scores to the bias terms itself.

Consider a positive logit $f(x)$ as a linear combination of the last layer activations $\phi^{(L)}$  with a non-positive bias $b \le 0$:
\begin{align}
    0 <& f(x) = \sum_i w_i \phi^{(L)}_i(x) +b \\
    0<&R\left( \sum_i w_i \phi^{(L)}_i(x) \right) = 
    \sum_i R\left( w_i \phi^{(L)}_i(x) \right) \label{eq:relstart}
\end{align}
We can see that the explanations for the last layer activations must sum to a positive value as well. Now let us consider the output of a ReLU feature \begin{align}
\phi^{(L)}_i(x) = ReLU\left( \sum_k w_k \phi^{(L-1)}_k(x) +b\right) \ .
\end{align} 
If the negative contributions to it dominate, then the output value of the ReLU is zero. This has the meaning that this neuron detects no feature. In this case $ R\left( w_i \phi^{(L)}_i(x) \right) =0$, and no explanation scores will be propagated back to its inputs $\phi^{(L-1)}_k(x)$, that is $R(\phi^{(L-1)}_k(x))=0$ received along this path from $\phi^{(L)}_i(x)$. 

If positive contributions to it dominate, then $0< \mathrm{ReLU}$  and we use the same idea as in the previous section:
\begin{align}
0<& \mathrm{ReLU}\left( \sum_k w_k \phi^{(L-1)}_k(x)\right)= \sum_k w_k \phi^{(L-1)}_k(x) \end{align} 

\begin{align}
\phi^{(L)}_i(x) &= ReLU\left( \sum_k w_k \phi^{(L-1)}_k(x) +b\right)\nonumber\\
\Rightarrow R \left(w_i\phi^{(L)}_i(x)\right) &= R\left(\mathrm{ReLU}\left( \sum_k w_k  \phi^{(L-1)}_k(x)\right)\right) \nonumber\\
&= \sum_k R(w_k\phi^{(L-1)}_k(x))
\end{align} 

\begin{align}
\Rightarrow 0 < \sum_i  R\left( w_i \phi^{(L)}_i(x) \right) &= \sum_i \sum_k R(w_k\phi^{(L-1)}_k(x))
\label{eq:relLminus1} 
\end{align} 
We use here only additivity of explanations $R(\cdot)$, and non-assignment of explanation scores to bias terms. In summary, combining equation \eqref{eq:relstart} with \eqref{eq:relLminus1} shows that the sum of relevances in layer $L-1$ is positive and equal to the initial logit relevance. Iterating this through all layers proves the claim until the input. In practice, explaining positive logits with methods which satisfy such an additivity, will result in dominantly positive explanations.

\section{Top-down Model Randomization based Experiments}
\label{ssec:topdownrandomexp}

Please see Figures \ref{fig:ssim2} and \ref{fig:mse} for the results. For better comparability all attribution maps were normalized by the square root of their second moment (not their variance) as discussed in Section \ref{sec:msesens}. The results are in principle known from \cite{DBLP:conf/nips/AdebayoGMGHK18}.

\begin{figure*}[t!]
\centering     
\subcaptionbox{ResNet50}{\label{fig:ssim2:a}
\includegraphics[width=0.31\linewidth]{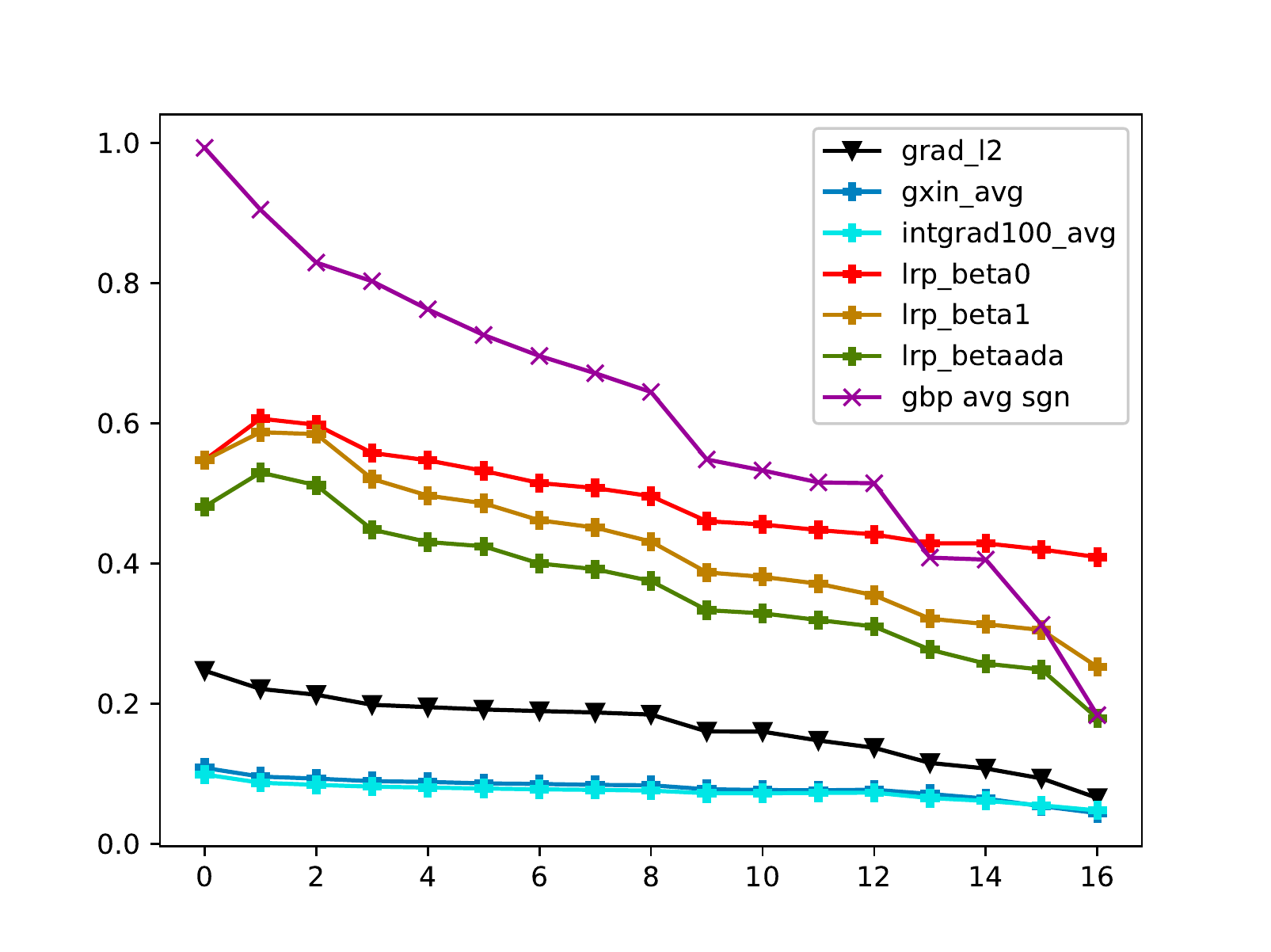}
}
\subcaptionbox{DenseNet121}{\label{fig:ssim2:b}
\includegraphics[width=0.31\linewidth]{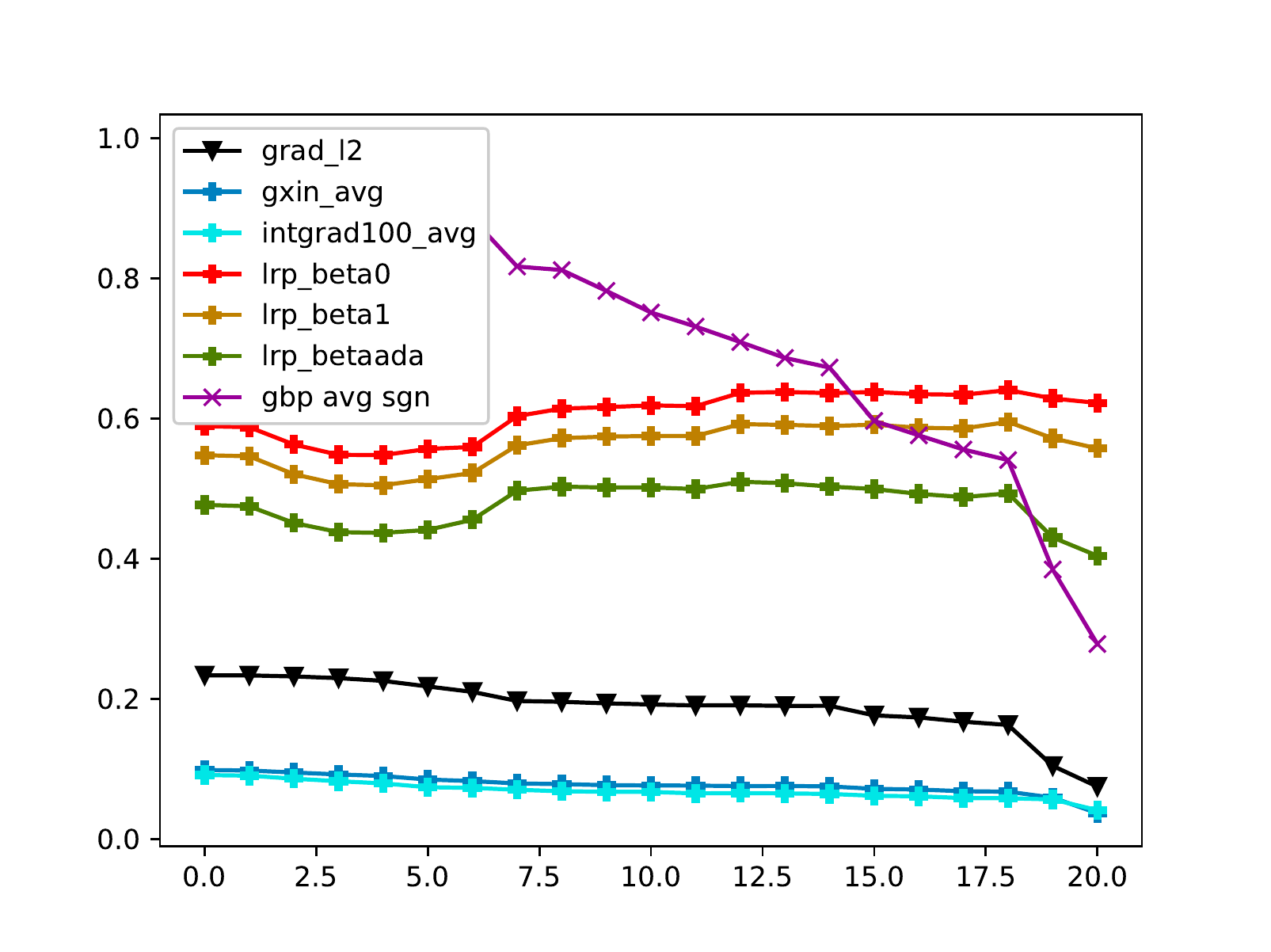}
}
\subcaptionbox{EfficientNet-B0}{\label{fig:ssim2:c}
\includegraphics[width=0.31\linewidth]{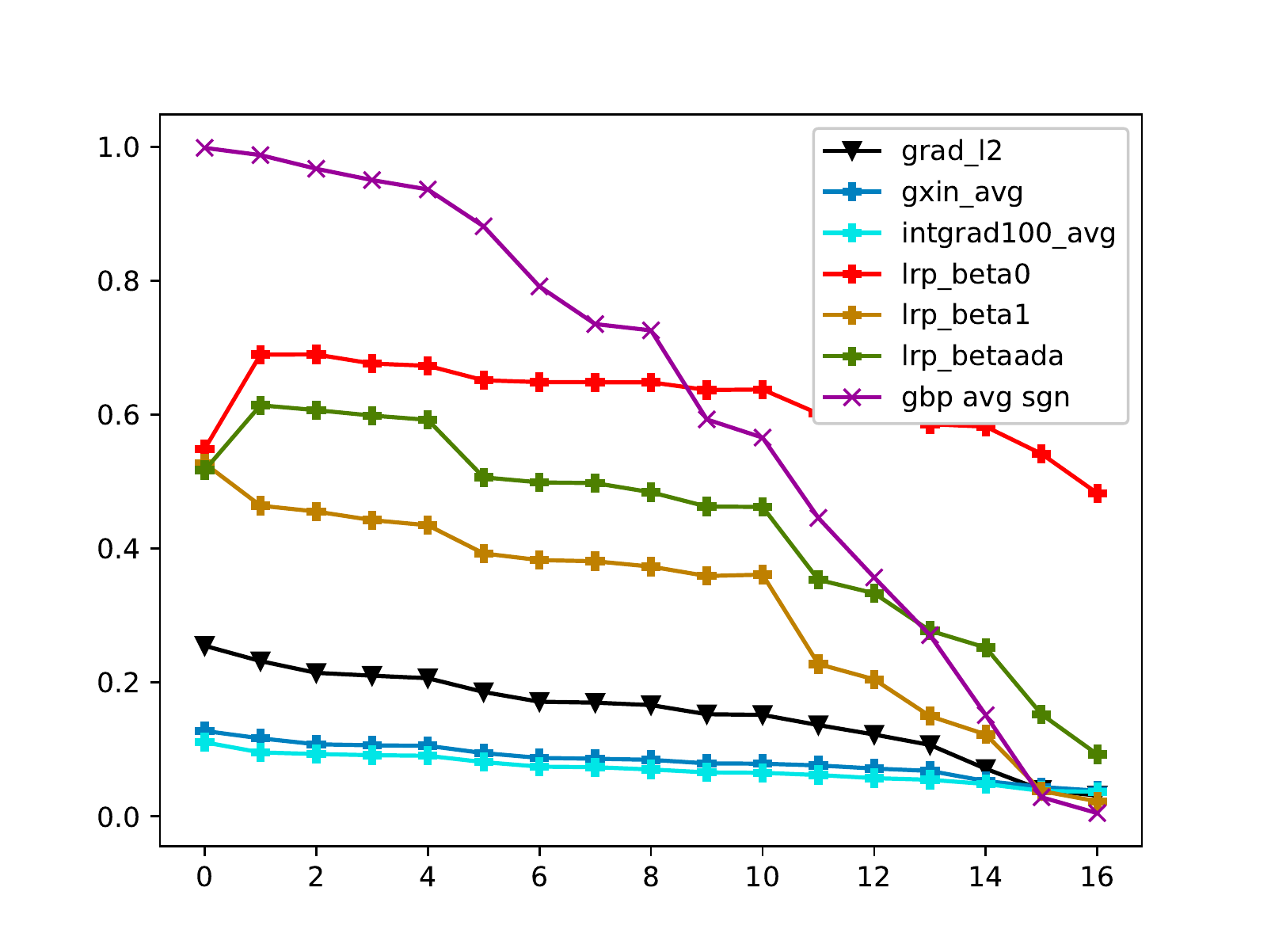}
}
\caption{\label{fig:ssim2} The figure shows the results of top-down model randomization-based sanity checks with \gls{ssim} after normalization of attribution maps by their second moment. \textit{Lower is better.}}
\end{figure*}

\begin{figure*}[t!]
\centering     
\subcaptionbox{ResNet50}{\label{fig:mse:a}
\includegraphics[width=0.31\linewidth]{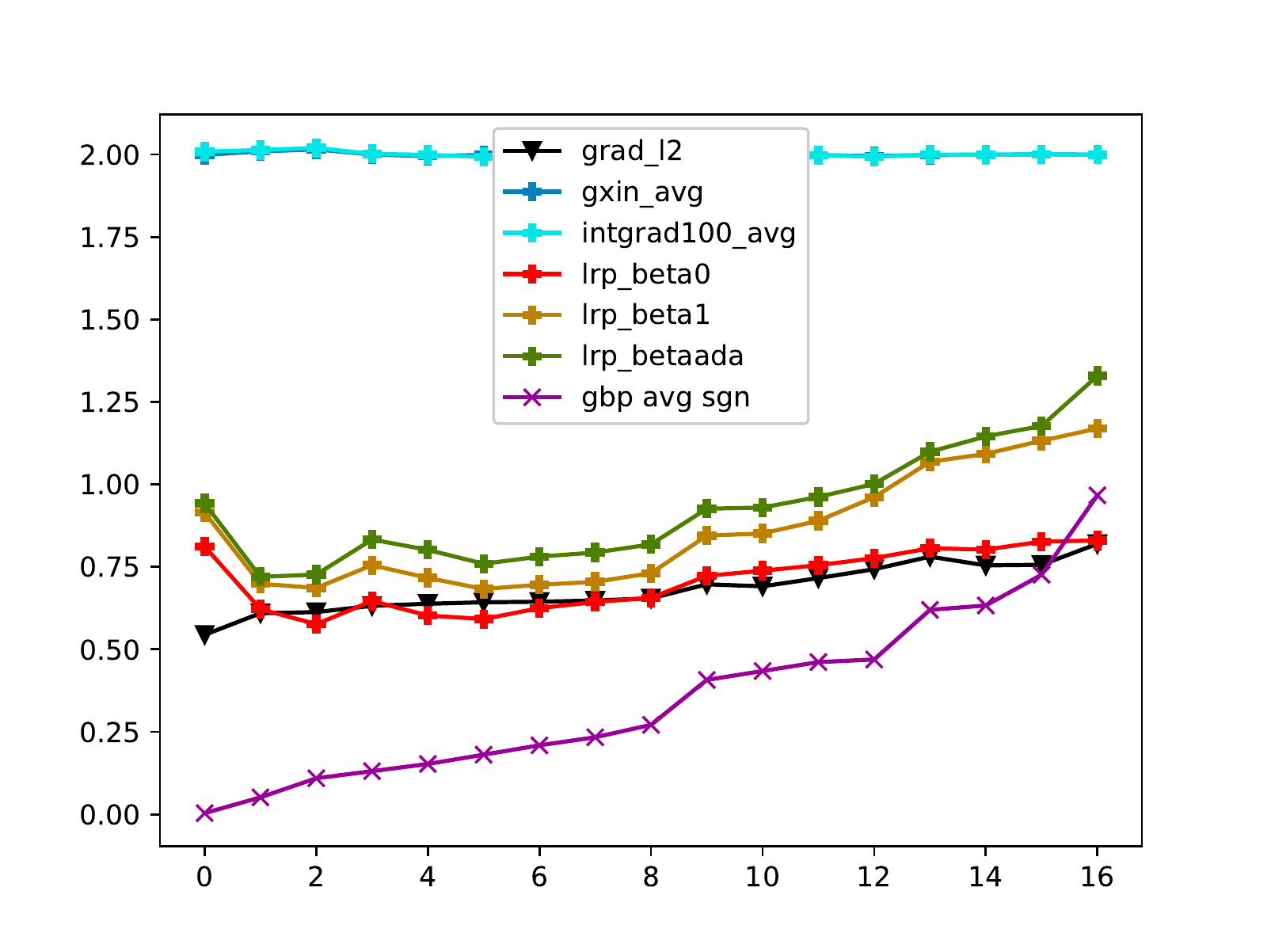}
}
\subcaptionbox{DenseNet121}{\label{fig:mse:b}
\includegraphics[width=0.31\linewidth]{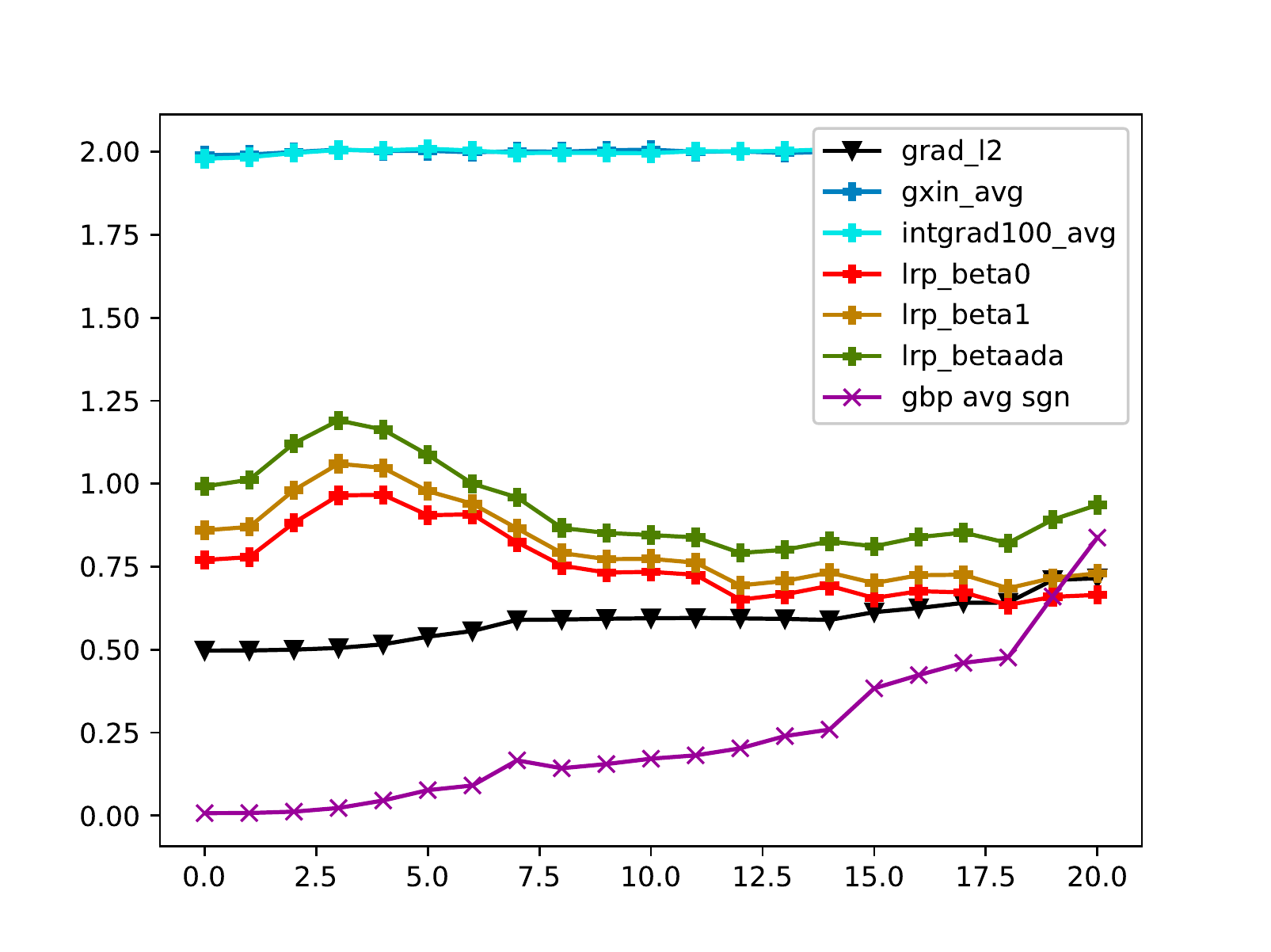}
}
\subcaptionbox{EfficientNet-B0}{\label{fig:mse:c}
\includegraphics[width=0.31\linewidth]{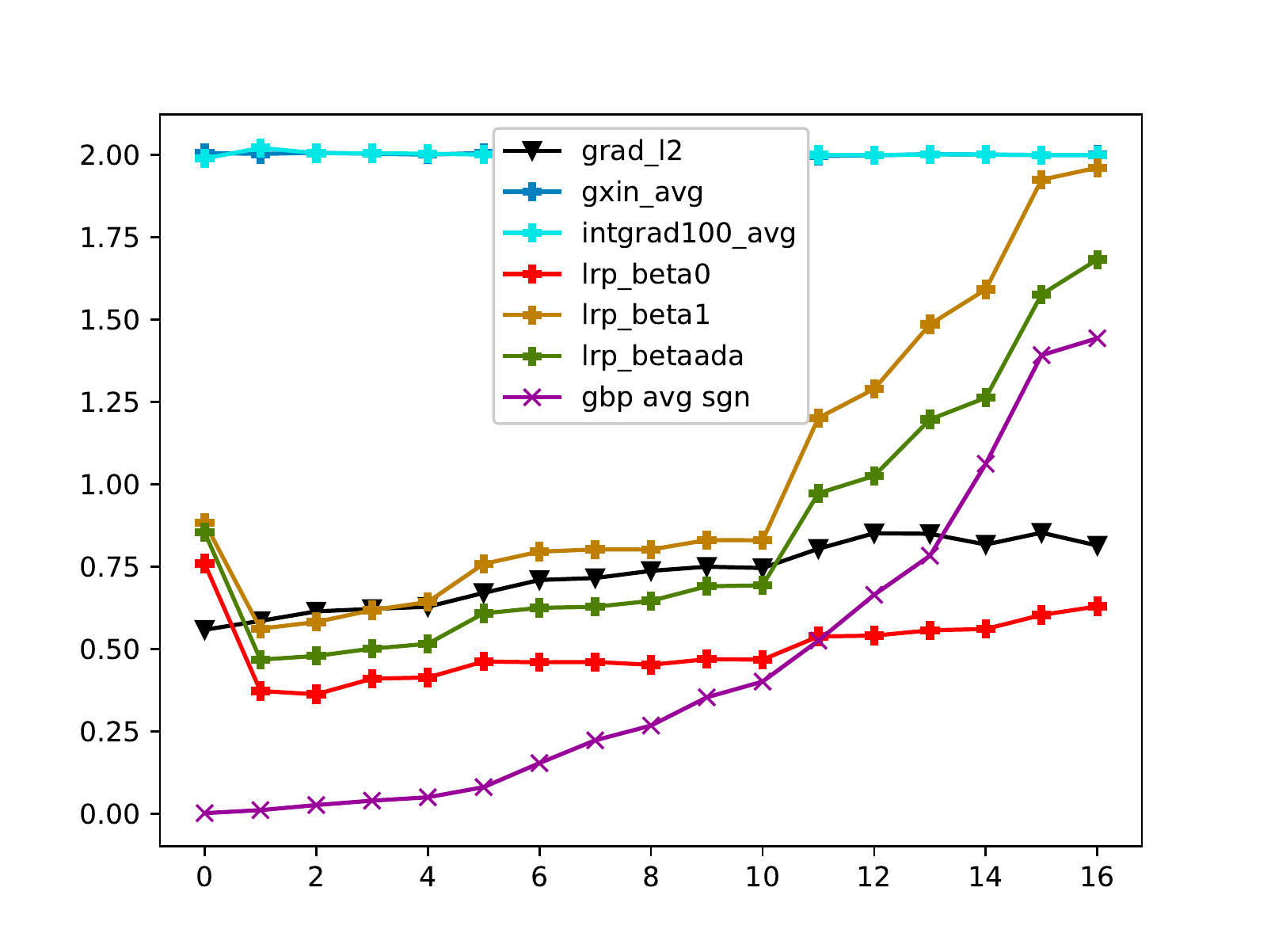}
}
\caption{\label{fig:mse} The figure shows the results of top-down model randomization-based sanity checks with \gls{mse} after normalization of attribution maps by their second moment. Of note is also the score of gradient-based results close to the value of $2$ in comparison with the upper bound in Equation \eqref{eq:normalizedmse}. \textit{Higher is better.}}
\end{figure*}

\section{Probabilities of overtaking large activations from forward pass activation statistics}
\label{ssec:probforwardpassovertaking}

This Section computes an upper bound for the probability of overtaking according to Theorem 2 for given trained models from Resnet-50, DenseNet-121 and EfficientNet-B0 architectures. This shows that in practice these probabilities are small.

To do this, we compute for a given image the forward pass activations, and pool them in every layer across spatial and channel dimensions (because usually a convolution kernel takes all channels as input). Next we compute a set of quantile estimators for these activation values in the range 
from $0.95$ down to $0.1$ in $0.05$ decrements. This yields 18 quantile estimators for every layer of the net and for one image. We compute the mean value of these estimators over 1000 images from the ImageNet validation set. 

After that we can compute estimates for the value of 
$\gamma=\sqrt{\frac{|X_S|}{|X_L|}}$
and $K$ from this information for every pair $(q_h,q_l)$ of a high quantile $q_h \in \{0.95,\ldots,0.85\}$ and a low quantile $q_l \in \{0.5,\ldots, 0.1\}$. $K$ is given as the ratio of quantile estimator values $K \ge \frac{V(q_h)}{V(q_l)} $, whereas $\gamma$ is given as 
$\sqrt{\frac{|X_S|}{|X_L|}} = \sqrt{\frac{q_l}{1-q_h}}$,
which corresponds to the relative fractions of the amount of bottom-k\% activations to the amount of observed top-k\% activations.   

Finally we can plug this into the Cauchy cumulative tail density $P_{\gamma}( Z \ge K)$ to obtain the probabilities. 

Each plot shows on the x-axis the low quantile $q_l \in \{0.5,\ldots, 0.1\}$, and on the y-axis $P_{\gamma}( Z \ge K)$. It shows one graph of probabilities $P_{\gamma}( Z \ge K)$ for each value of the high quantile $q_h \in \{0.95,0.9,0.85\}$. The graphs are color coded according to $q_h$.

The results are shown in Figures \ref{overtakingresnets1}, \ref{overtakingdensenets1} and \ref{overtakingeffnets1}.

We can see rather low probabilities 
despite the Cauchy distribution having a low order polynomial decay. 
Note that the EfficientNet can have negative activation statistics for some lower layers. In this case $K$ is computed using the inverse (because in this case one wants to overtake the absolute larger negative values using the absolute smaller negative values).

Some graphs, like for Resnet-50 levels 9 and 12 remain almost flat zero because the mean activation is very close to zero for the bottom-50\% values due to a strong sparsity in these layers. See Section \ref{sec:actstats} for the fraction of non-positive activations as an explanation, and compare the graph against Resnet-50 Level 6 and the corresponding statistics in Section \ref{sec:actstats}. We have verified that for higher bottom-\% values one would see small positive overtaking probabilities $P_{\gamma}( Z \ge K)$.

\begin{figure*}[t!]
\centering     
\subcaptionbox{ResNet50 Level0}{
\includegraphics[width=0.31\linewidth]{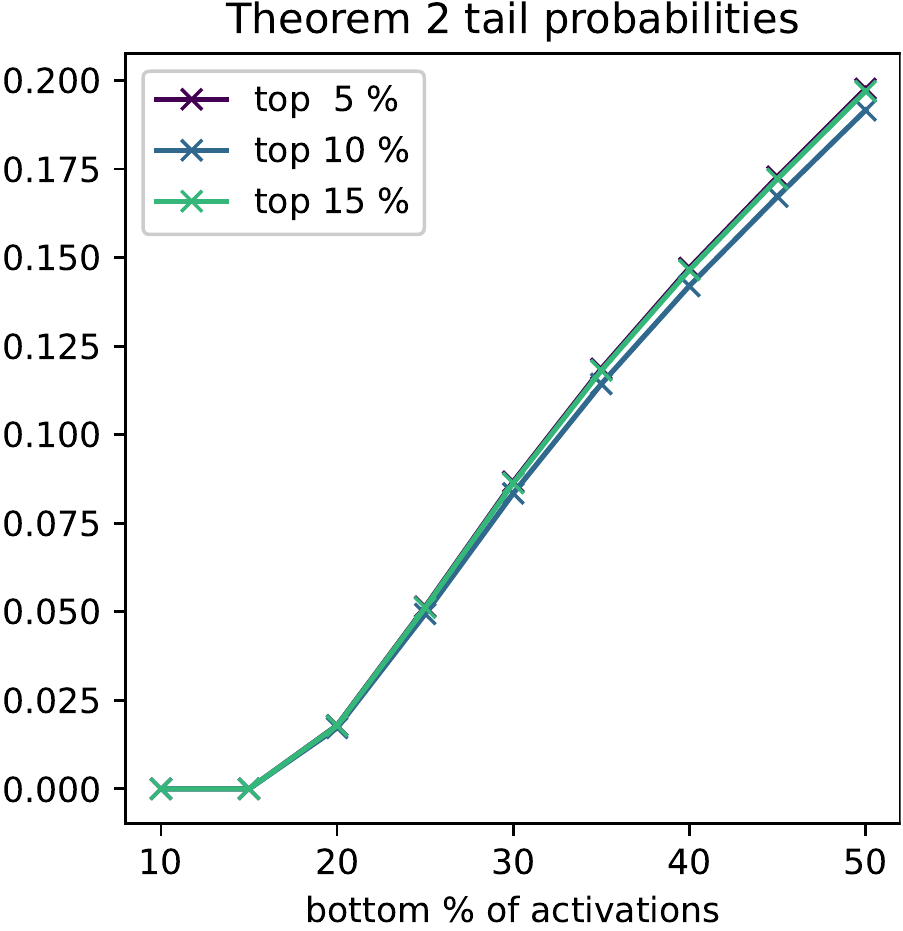}
}
\subcaptionbox{ResNet50 Level3}{
\includegraphics[width=0.31\linewidth]{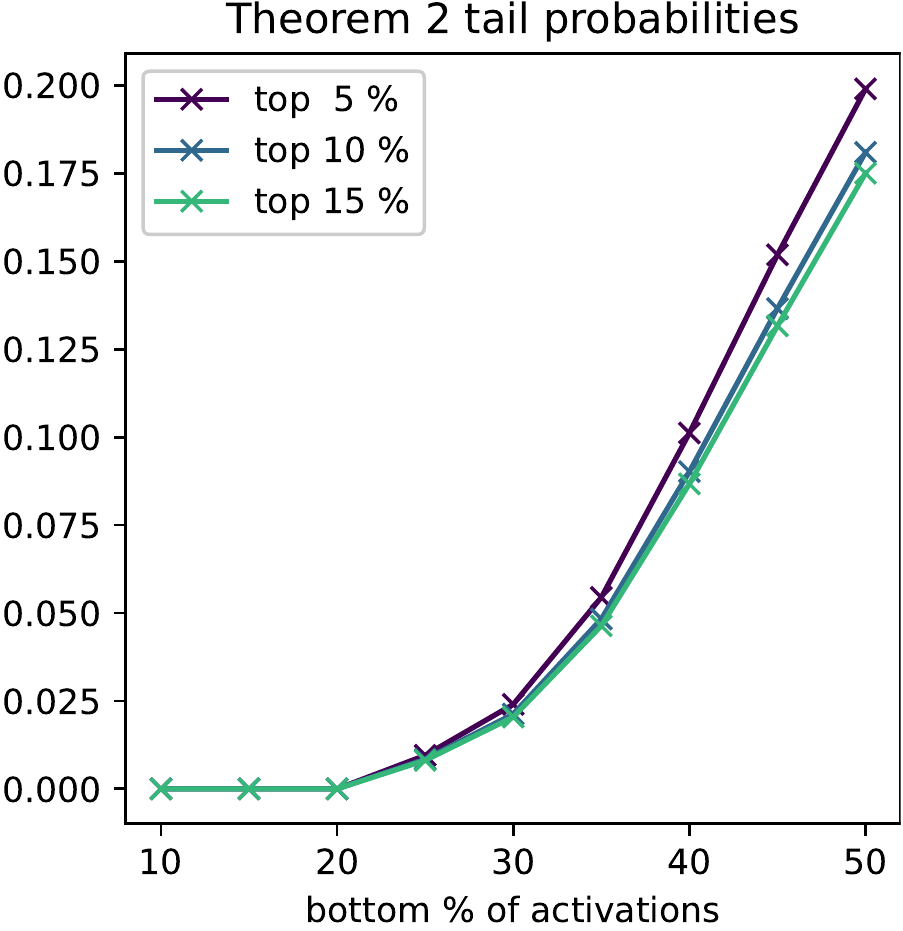}
}
\subcaptionbox{ResNet50 Level6}{
\includegraphics[width=0.31\linewidth]{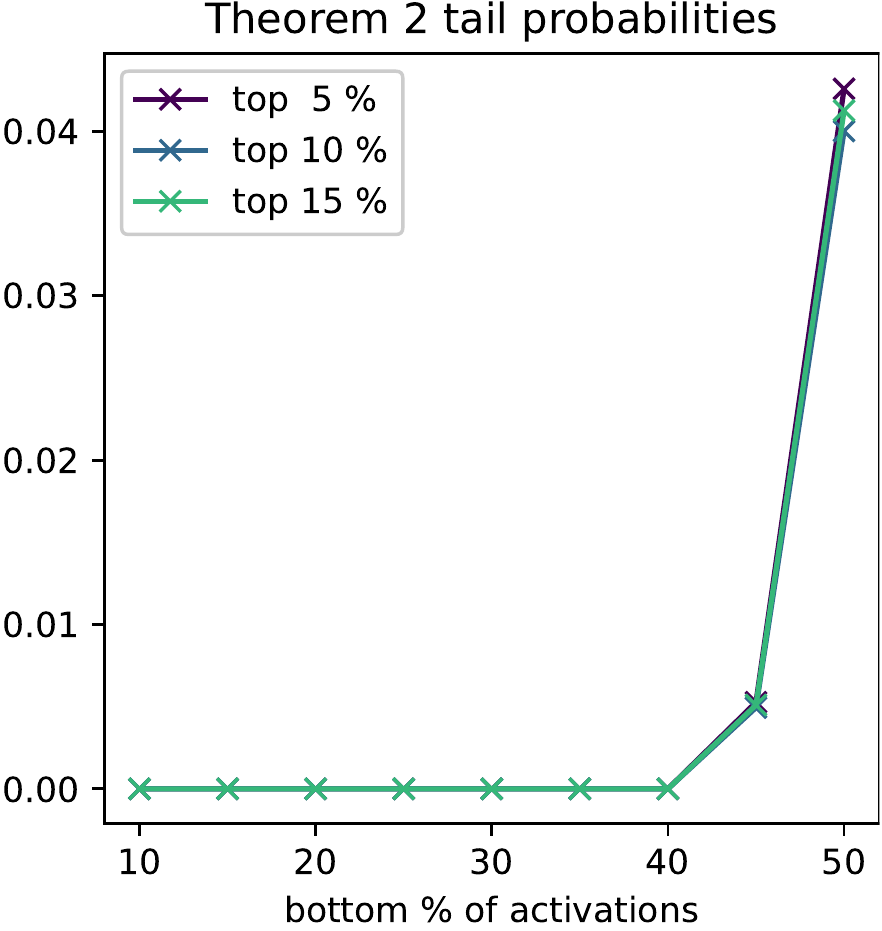}
}

\subcaptionbox{ResNet50 Level9}{
\includegraphics[width=0.31\linewidth]{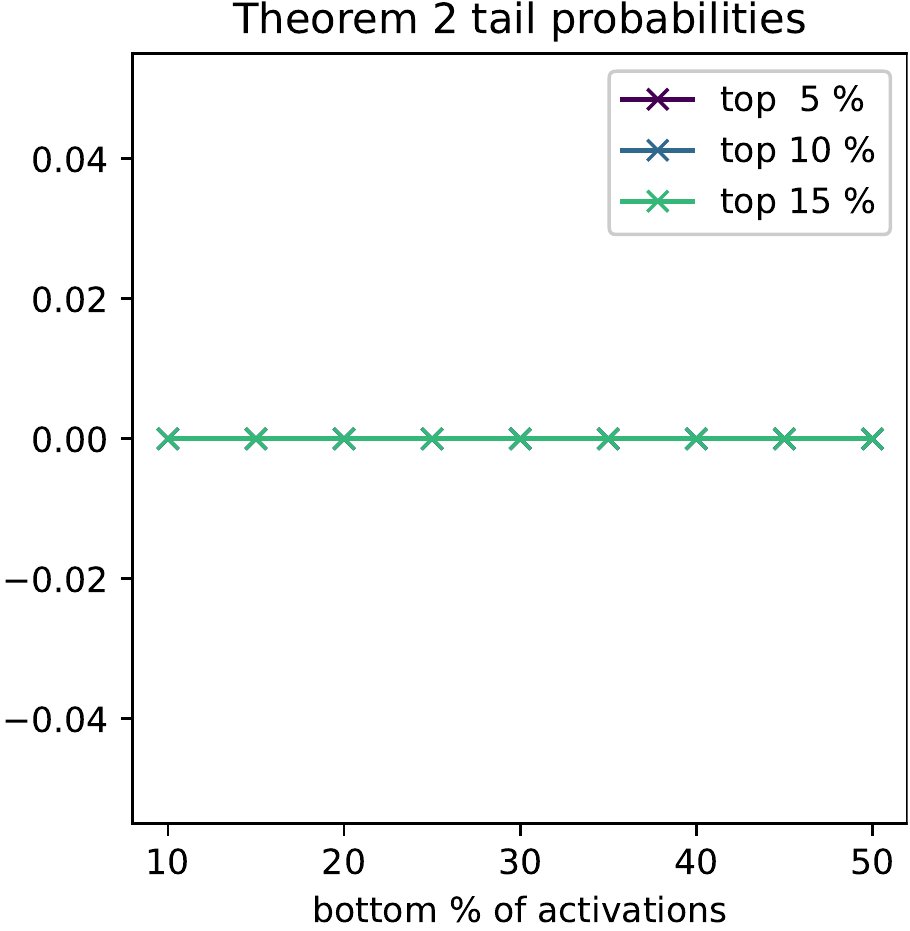}
}
\subcaptionbox{ResNet50 Level12}{
\includegraphics[width=0.31\linewidth]{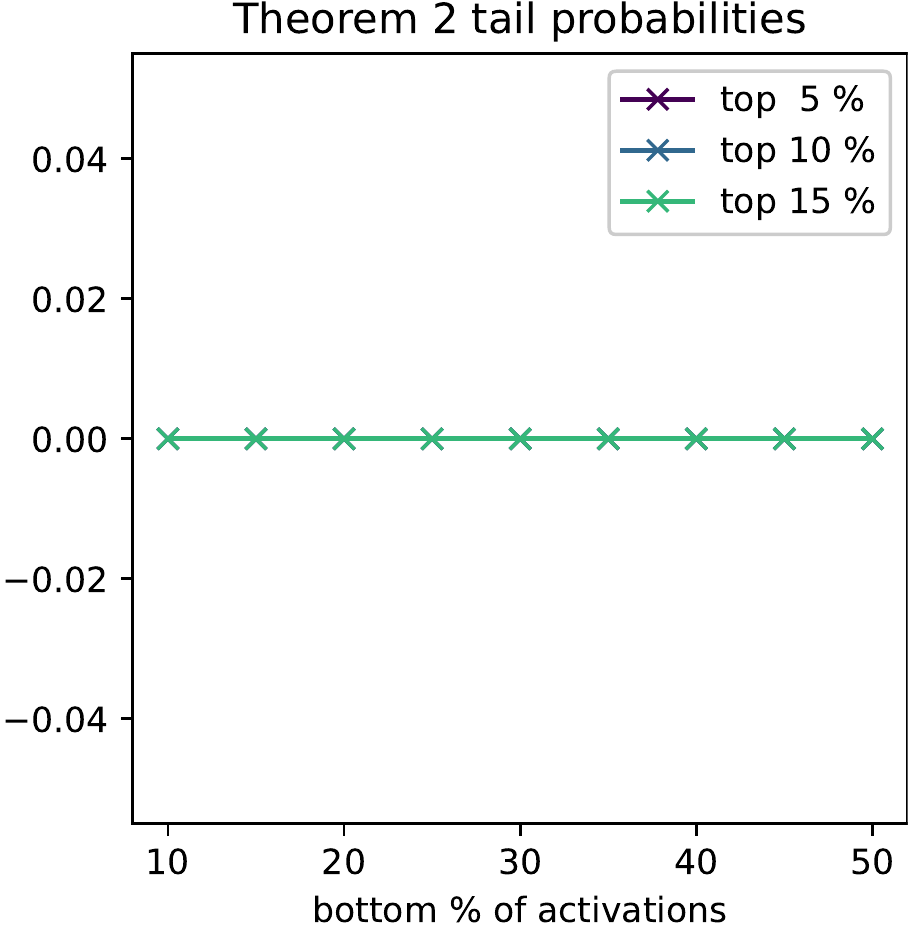}
}
\subcaptionbox{ResNet50 Level16}{
\includegraphics[width=0.31\linewidth]{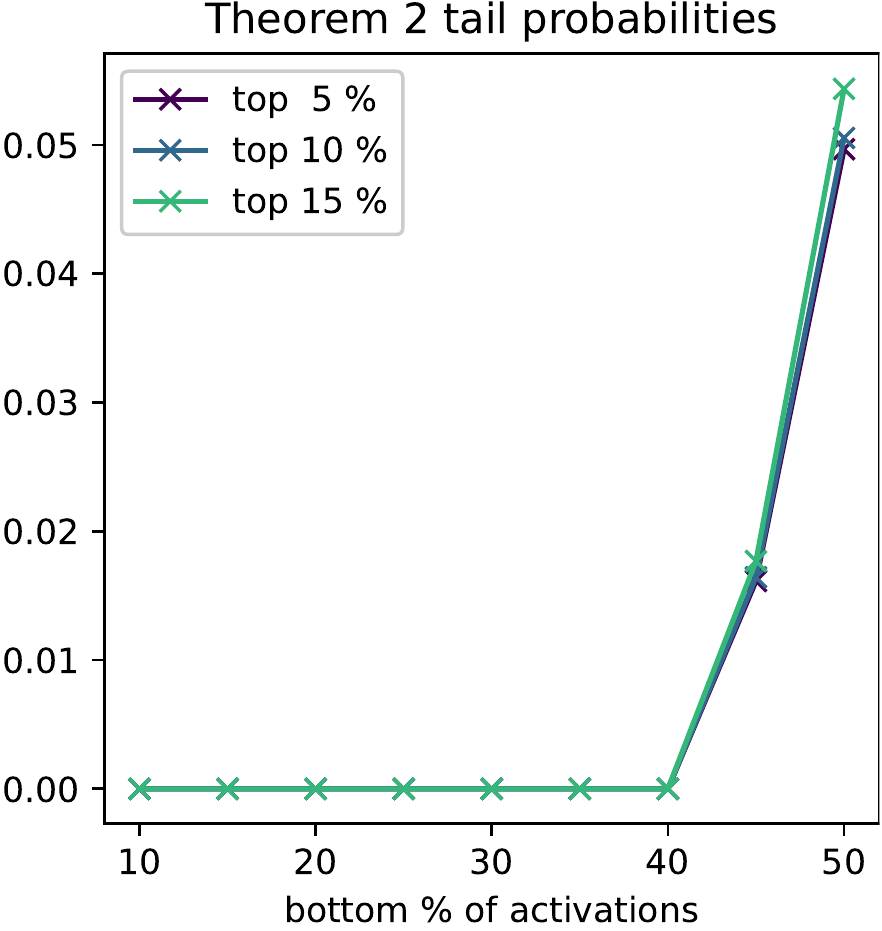}
}

\caption{\label{overtakingresnets1} Lower probabilities support Theorem 2 better.}
\end{figure*}

\begin{figure*}[t!]
\centering     
\subcaptionbox{DenseNet121 Level0}{
\includegraphics[width=0.31\linewidth]{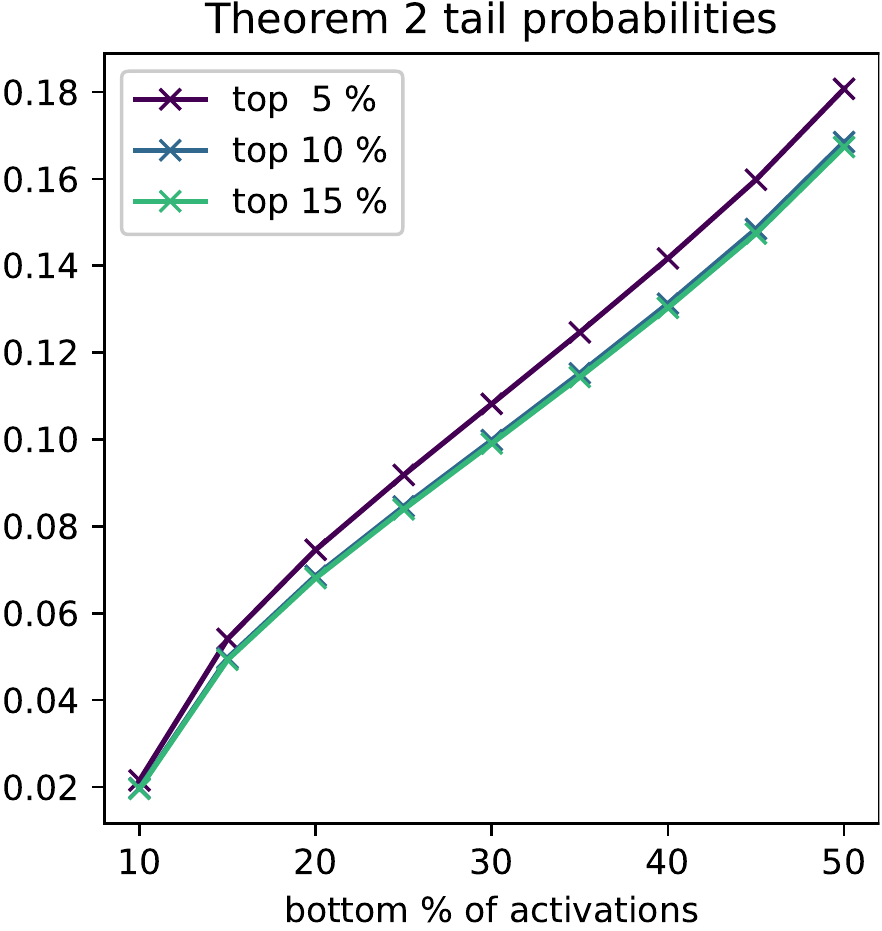}
}
\subcaptionbox{DenseNet121 Level10}{
\includegraphics[width=0.31\linewidth]{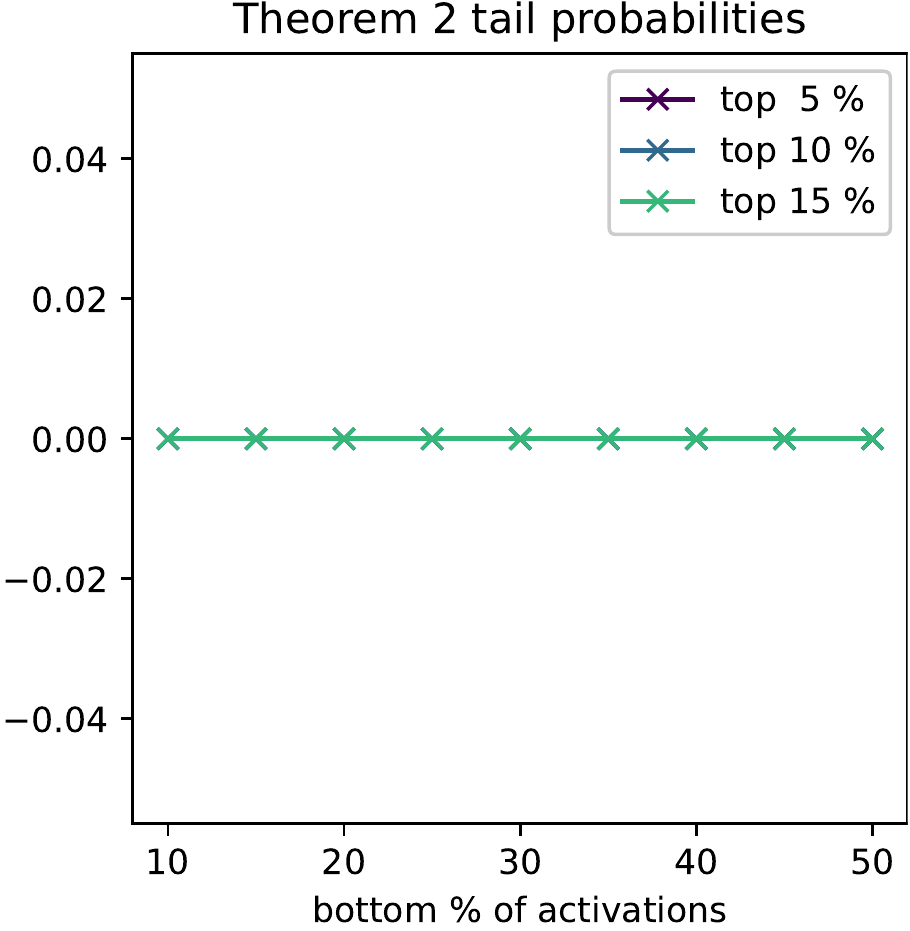}
}
\subcaptionbox{DenseNet121 Level20}{
\includegraphics[width=0.31\linewidth]{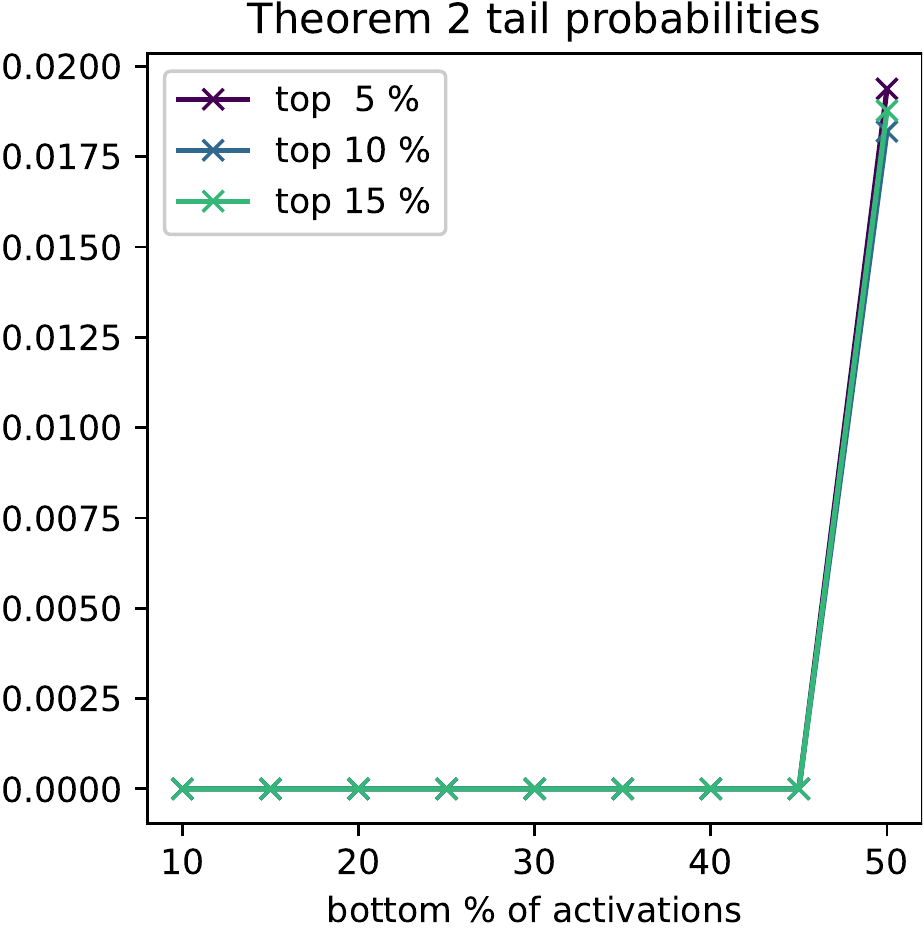}
}

\subcaptionbox{DenseNet121 Level30}{
\includegraphics[width=0.31\linewidth]{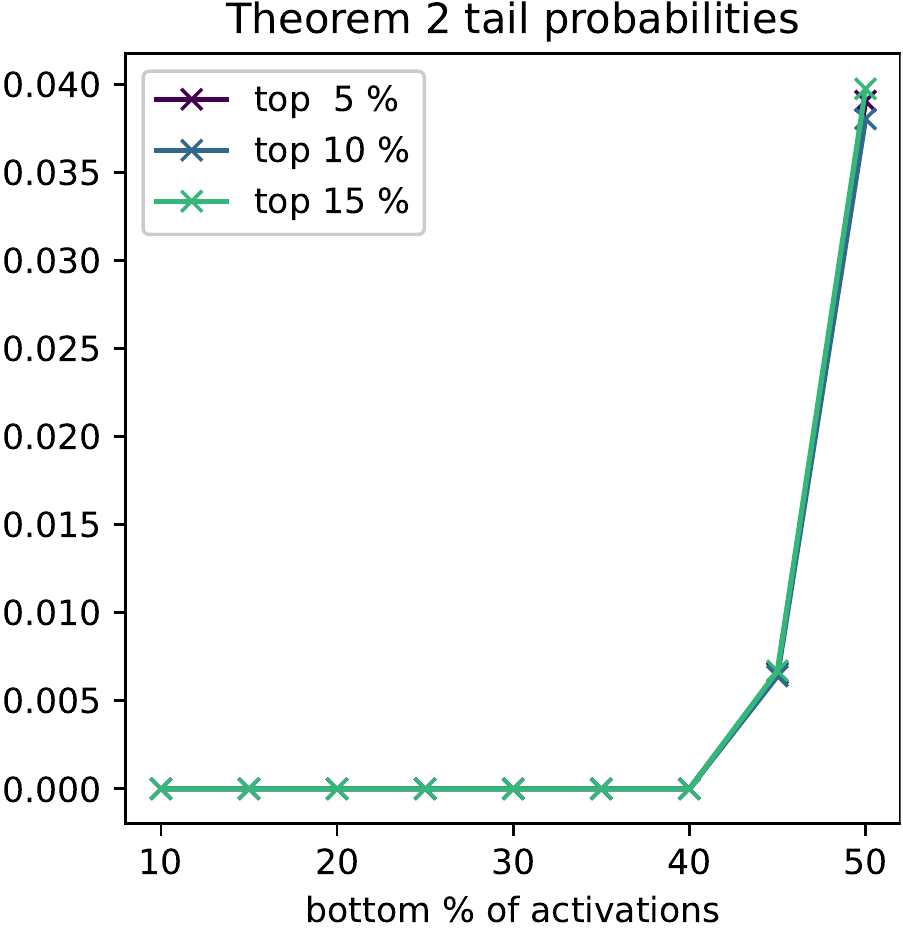}
}
\subcaptionbox{DenseNet121 Level73}{
\includegraphics[width=0.31\linewidth]{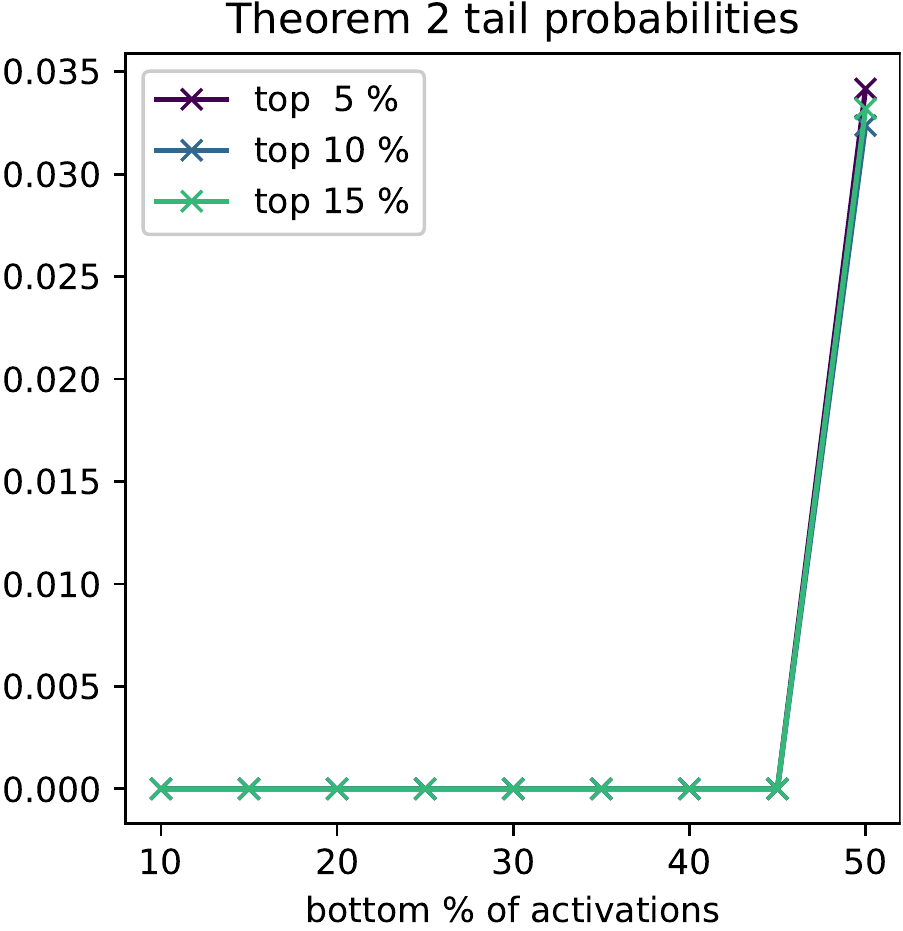}
}
\subcaptionbox{DenseNet121 Level74}{
\includegraphics[width=0.31\linewidth]{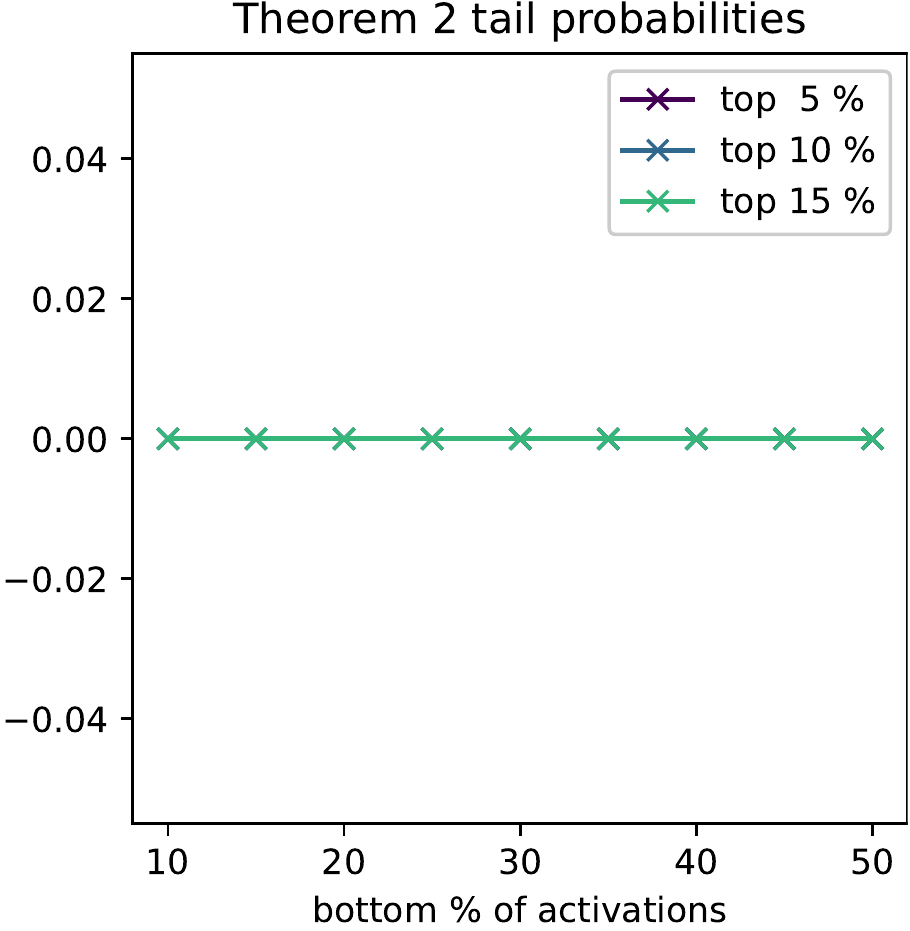}
}

\caption{\label{overtakingdensenets1} Lower probabilities support Theorem 2 better.}
\end{figure*}

\begin{figure*}[t!]
\centering     
\subcaptionbox{EfficientNet-B0 Level0}{
\includegraphics[width=0.31\linewidth]{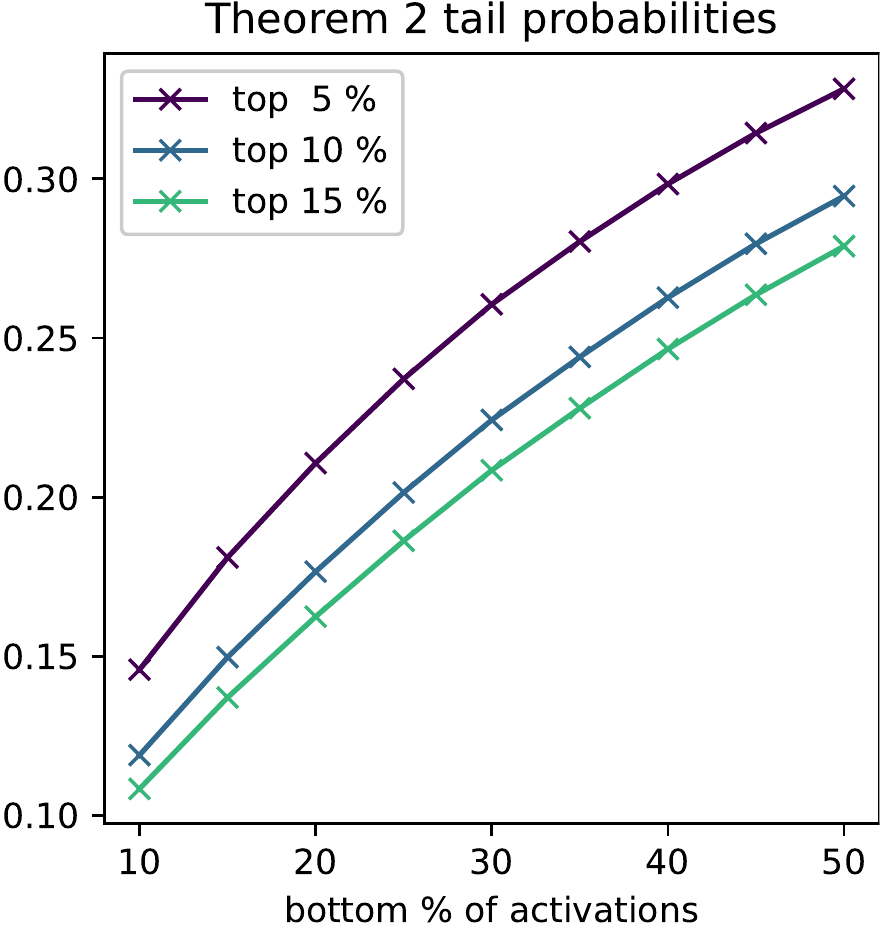}
}
\subcaptionbox{EfficientNet-B0 Level3}{
\includegraphics[width=0.31\linewidth]{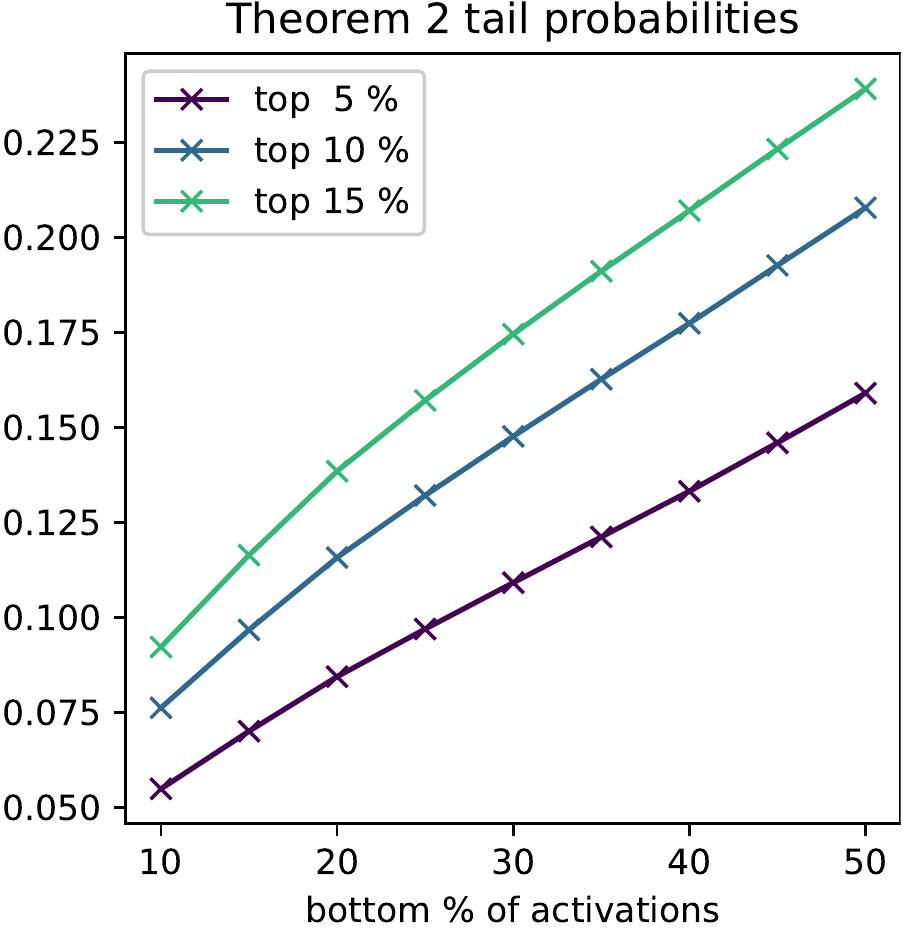}
}
\subcaptionbox{EfficientNet-B0 Level6}{
\includegraphics[width=0.31\linewidth]{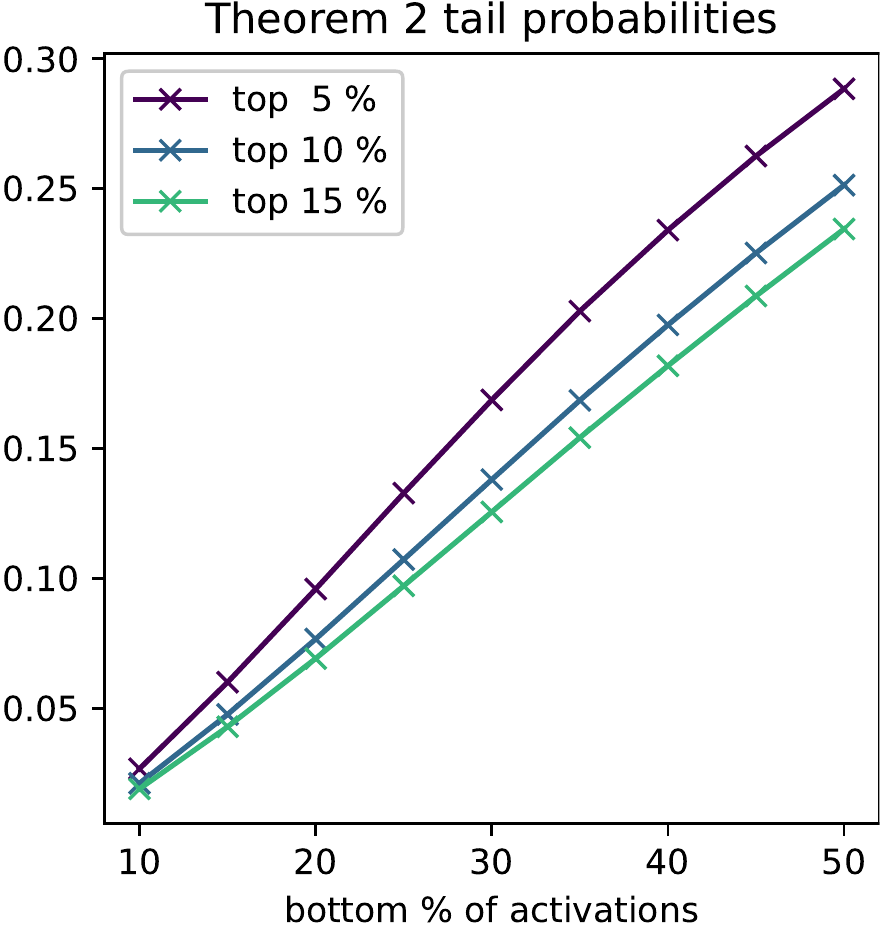}
}

\subcaptionbox{EfficientNet-B0 Level9}{
\includegraphics[width=0.31\linewidth]{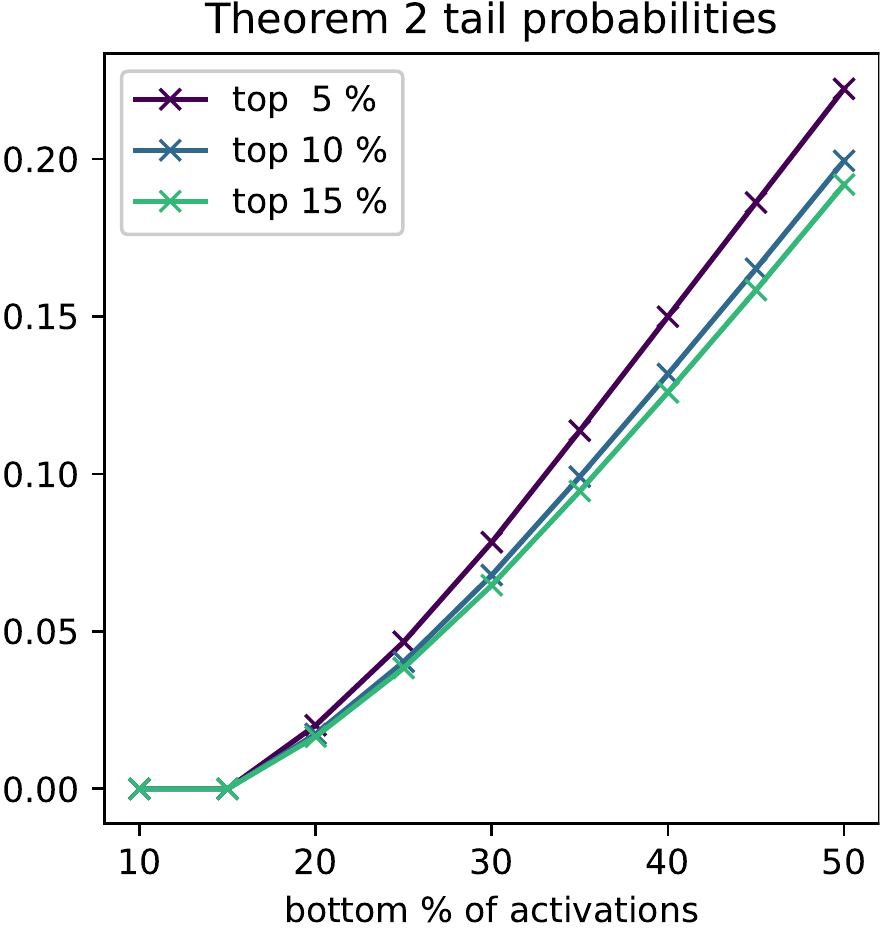}
}
\subcaptionbox{EfficientNet-B0 Level12}{
\includegraphics[width=0.31\linewidth]{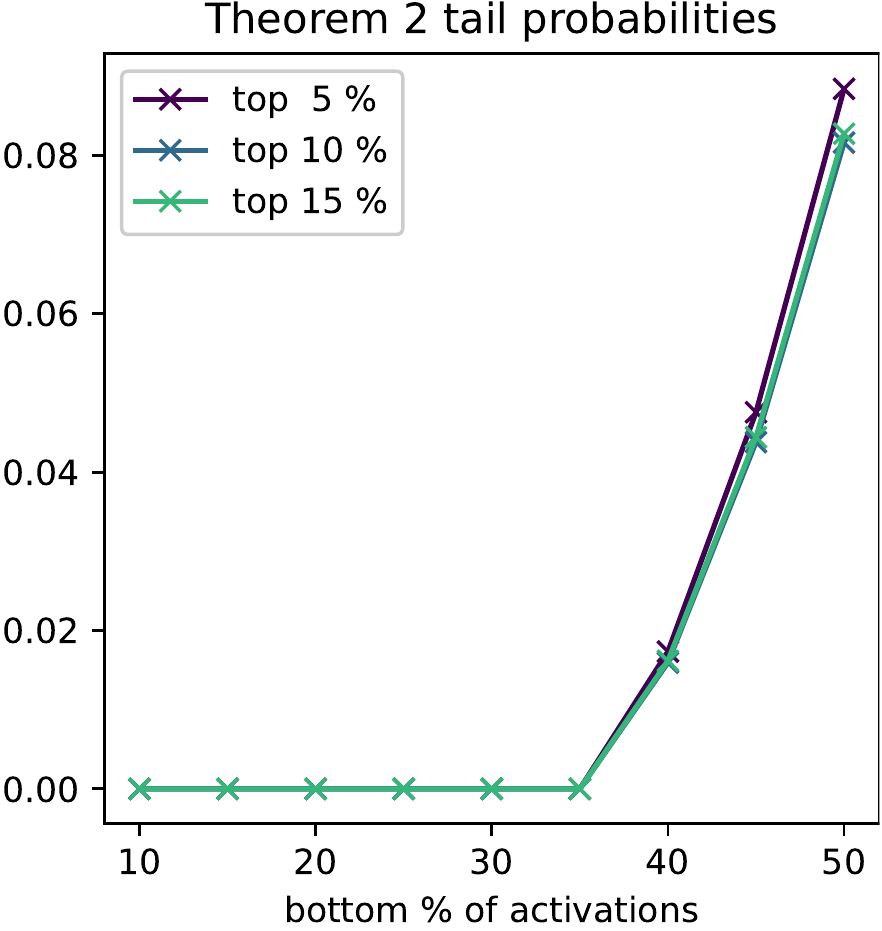}
}
\subcaptionbox{EfficientNet-B0 Level16}{
\includegraphics[width=0.31\linewidth]{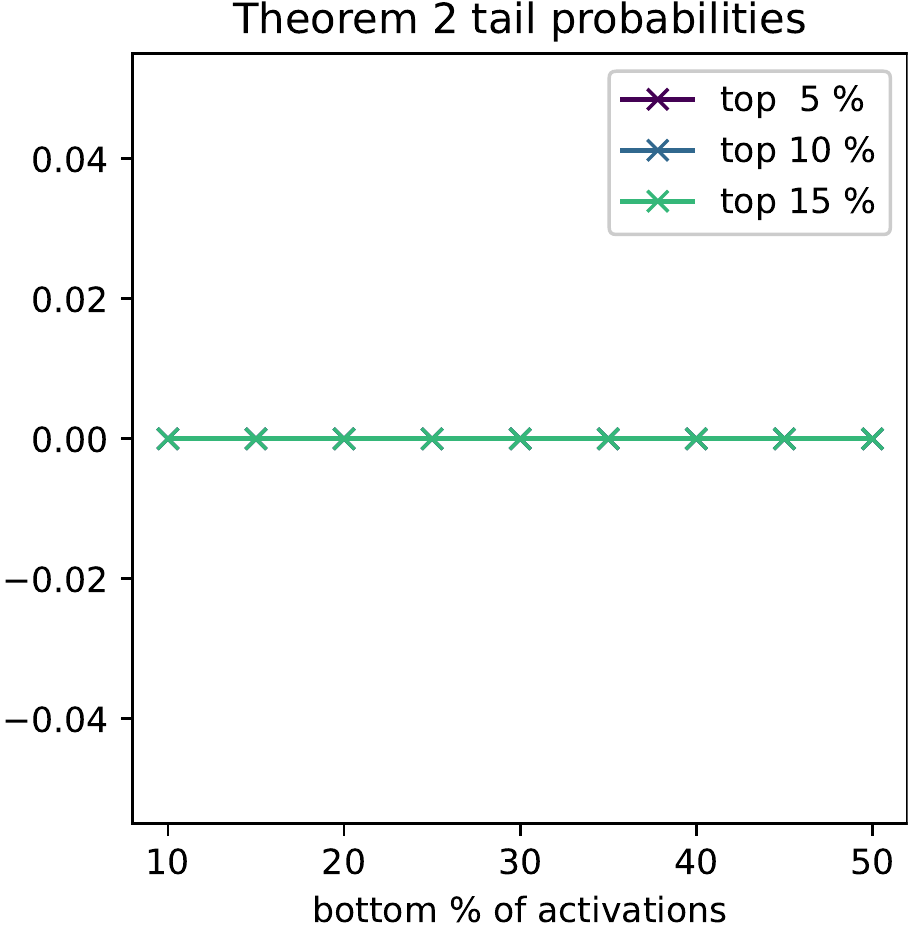}
}

\caption{\label{overtakingeffnets1} Lower probabilities support Theorem 2 better.}
\end{figure*}

\section{Activation Statistics}
\label{sec:actstats}

This section shows the fraction of non-positive activations. Results are shown in Figure \ref{fracnonpos}. One can see that for ResNet-50 and DenseNet-121, most layers have at least 30\% zero activations. The amount of nonpositive activations is less for the EfficientNet-B0, which makes sense as this is less wide than other architectures. From layer 11 onwards it has also at least 20\% zeros. Note that activations can get truly negative for the Efficientnet as a result of using the Swish activation function. Therefore seeing 100\% in layer 0,1,3,5 is not a mistake.

\begin{figure*}[htb]

\subcaptionbox{ResNet50}{
\includegraphics[width=0.3\linewidth]{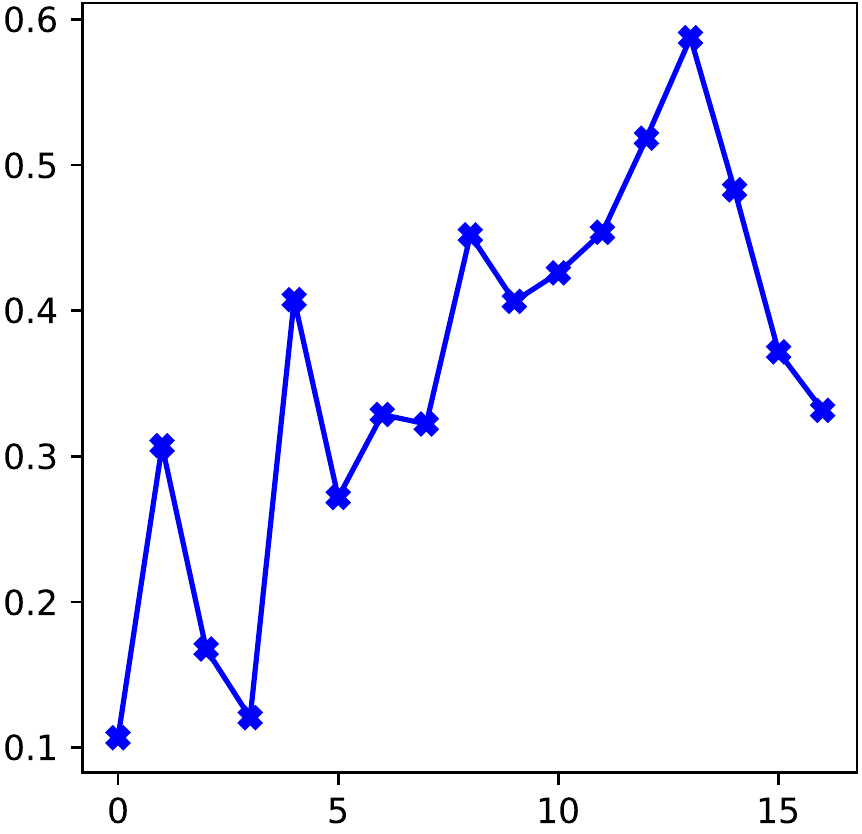}
}
\subcaptionbox{DenseNet-121}{
\includegraphics[width=0.3\linewidth]{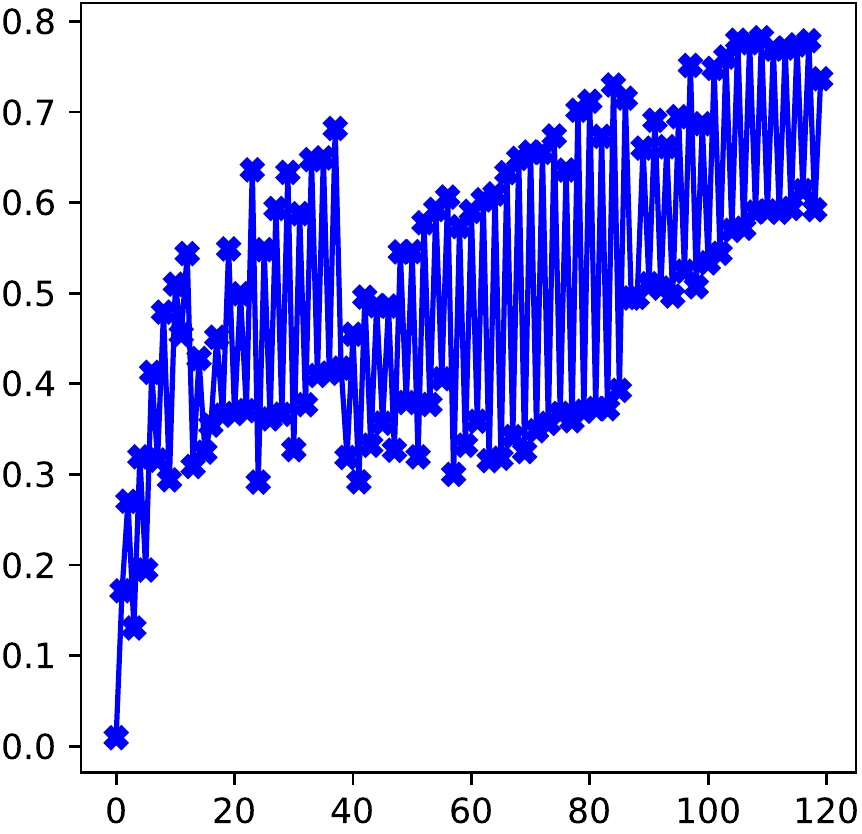}
}
\subcaptionbox{EfficientNet-B0}{
\includegraphics[width=0.3\linewidth]{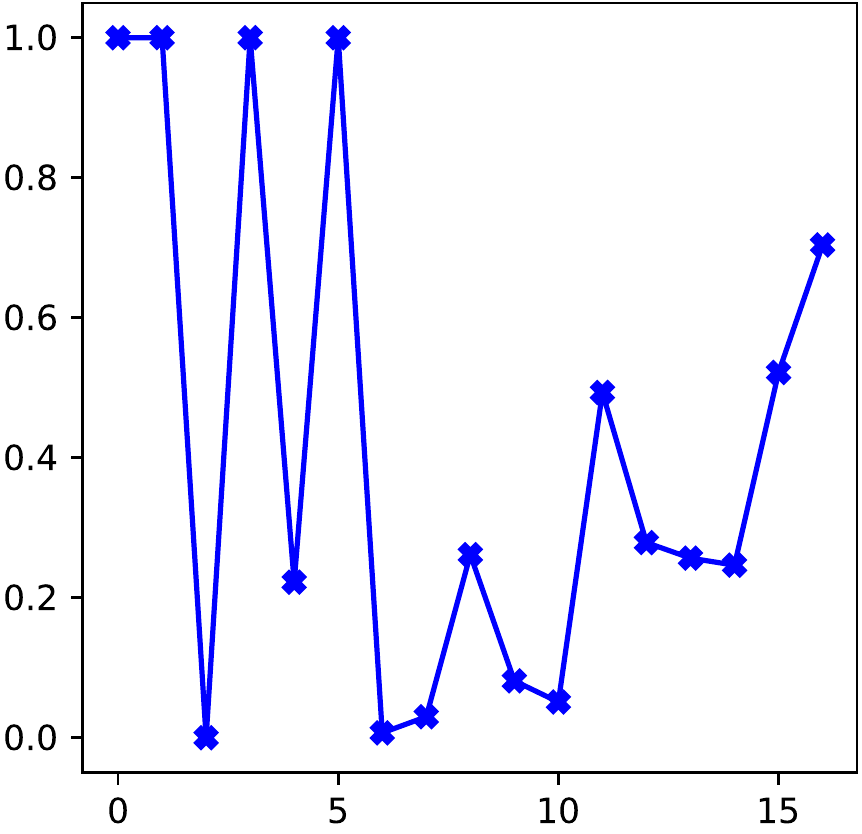}
}
\caption{\label{fracnonpos} Non-positive activations per layer. Higher values indicate a higher fraction of non-positive activations.}
\end{figure*}

\end{document}